\documentclass[10pt,journal,compsoc]{IEEEtran}
\usepackage{xcolor,amsmath,bbm,multirow,bbold,mathtools,amssymb,graphics,graphicx,epstopdf,ntheorem,diagbox}
\usepackage[square,sort,comma,numbers]{natbib}
\newcommand{\rF}{\rm F}
\def\##1\#{\begin{align}#1\end{align}}
\def\$#1\${\begin{align*}#1\end{align*}}
\newcommand{\argmin}{\mathop{\mathrm{argmin}}}

\usepackage[ruled,linesnumbered]{algorithm2e}

\newtheorem{theorem}{Theorem}

\newtheorem{proposition}[theorem]{Proposition}
\newtheorem*{proof}{Proof}
\newtheorem{lemma}[theorem]{Lemma}
\ifCLASSOPTIONcompsoc
  \usepackage[nocompress]{cite}
\else
  \usepackage{cite}
\fi
\ifCLASSINFOpdf
\else
\fi
\usepackage{color}

\begin{document}
%
\title{Quadratic Matrix Factorization with Applications to Manifold Learning}
%
%
%
%

\author{Zheng~Zhai,
        Hengchao~Chen,
        and~Qiang Sun
\IEEEcompsocitemizethanks{
\IEEEcompsocthanksitem Zheng Zhai, Hengchao Chen, and Qiang Sun are with the Department of Statistical Sciences,
       University of Toronto\\ 
\IEEEcompsocthanksitem Corresponding author: Qiang Sun. E-mail: qiang.sun@utoronto.ca
}
\thanks{Manuscript received April 19, 2005; revised August 26, 2015.}}

%
%

\markboth{Journal of \LaTeX\ Class Files,~Vol.~14, No.~8, August~2015}%
{Shell \MakeLowercase{\textit{et al.}}: Bare Demo of IEEEtran.cls for Computer Society Journals}
%



\IEEEtitleabstractindextext{%
\begin{abstract}
Matrix factorization is a popular framework for modeling low-rank data matrices. 
Motivated by manifold learning problems, this paper proposes a quadratic matrix factorization (QMF) framework to learn the curved manifold on which the dataset lies. Unlike local linear methods such as the local principal component analysis, QMF can better exploit the curved structure of the underlying manifold. Algorithmically, we propose an alternating minimization algorithm to optimize QMF and establish its theoretical convergence properties. Moreover, to avoid possible over-fitting, we then propose a regularized QMF algorithm and discuss how to tune its regularization parameter. Finally, we elaborate how to apply the regularized QMF to manifold learning problems. 
Experiments on a synthetic manifold learning dataset and two real datasets, including the MNIST handwritten dataset and a cryogenic electron microscopy dataset, demonstrate the superiority of the proposed method over its competitors. 
\end{abstract}

}

\maketitle

\IEEEdisplaynontitleabstractindextext

%
\IEEEpeerreviewmaketitle

\IEEEraisesectionheading{\section{Introduction}\label{sec:1}}

%
%
%
%

\IEEEPARstart{M}{atrix} factorization has achieved many successes in various applications, including factor models \citep{paatero1994positive}, clustering \citep{liu2013multi,wang2022robust}, recommendation system \citep{koren2009matrix}, graph and representation learning \citep{hamilton2017representation}. 
The key idea behind matrix factorization is that any data matrix admitting a low-rank structure can be represented as a product of two matrices with smaller dimensions.   
In a general form,  matrix factorization solves  
\begin{equation}\label{equ:1}
\min_{U\in {\mathbb R}^{D\times r}, {V\in {\mathbb R}^{r\times m} }
\atop 
U, V\in \mathcal{C}} \|X-UV\|_{\rF}^2.
\end{equation}
where $X\in \mathbb{R}^{D\times m}$ is the data matrix with dimension $D$ and sample size $m$, $U\in\mathbb{R}^{D\times r}$ and $V\in\mathbb{R}^{r\times m}$ are factors of rank $r$, $\mathcal{C}$ is the feasible set encoding additional structural information, and $\|\cdot\|_{\rF}$ is the Frobenius norm. 
Building upon \eqref{equ:1}, many well-known algorithms in the literature can be obtained by taking specific feasible sets $\mathcal{C}$.  
For example, if $\mathcal{C} =\{V: VV^T=I_r\}$ enforces orthonormal constraints on the rows of $V$, then \eqref{equ:1} reduces to principal component analysis (PCA). If $\mathcal{C}$ enforces non-negative constraints on both $U$ and $V$, or either $U$ or $V$, then \eqref{equ:1} becomes  nonnegative matrix factorization (NMF) \citep{lee2000algorithms} or semi-NMF \citep{ding2008convex} respectively. NMF also shares a strong connection to spectral clustering \citep{ding2005equivalence}. 


Although matrix factorization \eqref{equ:1} has been widely studied in various forms, it is less explored in the nonlinear case: some rows of $V$ might be nonlinear functions of other rows of $V$. 
Nonlinear matrix factorization naturally arises in manifold learning problems,
where data $\{x_i\}_{i=1}^m\subseteq{\mathbb R}^D$ are  assumed to concentrate near an unknown $d$-dimensional manifold $\cal M$ and the goal is to recover $\cal M$ from $\{x_i\}_{i=1}^m$~\citep{aamari2019nonasymptotic,aizenbud2021non}. 
Mathematically, assume $\{x_i\}_{i=1}^m$ are given by 
\begin{equation*}
x_i=f(\tau_i)+\epsilon_i,\quad  i=1,\ldots,m,
\end{equation*}
where $f$ is a bijective smooth mapping from an open set ${\cal T}\subseteq\mathbb{R}^d$ to ${\cal M}\subseteq\mathbb{R}^D$, $\tau_i\in{\cal T}$ is the $d$-dimensional representation of $x_i$, and $\epsilon_i$ is the $i$-th  approximation error. Here $f$ and $\tau_i$ are non-identifiable in the sense that $f(\tau_i)=f'(\tau'_i)$, where $f'=f\circ g$ and $\tau'_i=g^{-1}(\tau_i)$ for any diffeomorphism $g$ on $\cal T$, but $f(\tau_i)$ is uniquely defined. To find $f(\tau_i)$, we propose to minimize the residual sum of squares with respect to $f$ and $\tau_{i}$:
\begin{equation}\label{equ:3a}
	\min_{f\in{\cal F},\{\tau_i\}_{i=1}^m\subseteq\mathbb{R}^d}\sum_{i=1}^{m}\|x_i-f(\tau_i)\|_{\rF}^2.
\end{equation}
Here ${\cal F}=\{f:\mathbb{R}^d\mapsto\mathbb{R}^D\}$ is a prediction function class. 
Let
$X=(x_1,\ldots,x_m)$, $\Phi=(\tau_1,\ldots,\tau_m)$, and $\Xi(f,\Phi)=(f(\tau_1),\ldots,f(\tau_m))$. Then \eqref{equ:3a} can be rewritten as
\#\label{equ:3}
\min_{f\in{\cal F},\Phi\in\mathbb{R}^{m\times d}}\|X - \Xi(f,\Phi)\|_{\rF}^2.
\#

If $\cal F$ consists of linear functions only, i.e., ${\cal F}=\{f\mid f(\tau)=A\tau,
A\in\mathbb{R}^{D\times d}\}$,
then the optimization problem \eqref{equ:3} reduces to the matrix factorization problem \eqref{equ:1} with $U=A$, $V=\Phi$, and $\mathcal{C}= \mathbb{R}^{D\times r}\times \mathbb{R}^{r\times m}$. This is referred to as linear matrix factorization (LMF).
The linear assumption on  $f$ can be too restrictive for manifold learning problems because it does not take the curved structure into account and thus is only applicable to model flat manifolds, i.e., manifolds that are locally isometric to the Euclidean spaces. 
For general manifolds, it is better to consider $f$ as quadratic functions. Higher-order polynomial functions for $f$ are also possible but tend to overfit noisy data, rendering poor generalization performance. 
For the reasons above, this paper focuses on the quadratic function class
\begin{equation}\label{equ:F}
\begin{gathered}
{\cal F}=\{f(\tau)=c+A\tau+{\cal B}(\tau,\tau): c\in\mathbb{R}^D,A\in\mathbb{R}^{D\times d},\\
\textnormal{ symmetric tensor }\mathcal{B}\in\mathbb{R}^{D\times d\times d}\}.
\end{gathered}
\end{equation}
For any quadratic function $f$, there exists a unique matrix $R\in\mathbb{R}^{D\times (2+3d+d^2)/2}$ such that $f(\tau)=R\xi(\tau)$, where
\begin{equation}\label{xi}
\begin{gathered}
\xi(\tau)=[1,\tau^T,\psi(\tau)^T]^T\\
\psi(\tau)= [\tau_{[1]}^2,\tau_{[1]}\tau_{[2]},...,\tau_{[1]}\tau_{[d]},\tau^2_{[2]},...,\tau_{[d]}^2]^T\in\mathbb{R}^{(d^2+d)/2}.
\end{gathered}
\end{equation}
Here $\tau_{[i]}$ denotes the $i$-th coordinate of $\tau$ and $\psi(\cdot)$ maps $\tau$ to a vector consisting of all quadratic and interaction terms of $\{\tau_{[i]}\}_{i=1}^d$.  Let $T(\Phi)=(\xi(\tau_1),\ldots,\xi(\tau_m))$. Then the optimization problem \eqref{equ:3} with the  quadratic function class \eqref{equ:F} reduces to
$\min_{R,\Phi}\|X-RT(\Phi)\|_{\rF}^2.$
To make $\Phi$ identifiable, we propose to solve the following optimization problem 
\#
\min_{R,\Phi
\atop
\Phi\Phi^T=I_d,\Phi{\bf 1}_m={\bf 0}
}\|X-RT(\Phi)\|_{\rF}^2.\label{equ:6}
\#
This is again a special case of the general matrix factorization problem \eqref{equ:1}, and we emphasize that  $T(\Phi)$ encodes an implicit constraint that the last $(d^2+d)/2$ rows of $T(\Phi)$ are quadratic functions of its second-to-$(1+d)$-th rows. Thus, we refer to  \eqref{equ:6} as quadratic matrix factorization~(QMF).

LMF has been widely studied \citep{lee2000algorithms,wang2012nonnegative,gillis2015exact,abdolali2021simplex}, but the proposed QMF is much more difficult and has received little attention.
In this paper, we propose to optimize the QMF problem \eqref{equ:6} with respect to $\Phi$ and $R$ alternatively.
Specifically, with $\Phi$ fixed, optimizing \eqref{equ:6} with respect to $R$ is equivalent to a linear regression problem. With $R$ fixed, minimizing \eqref{equ:6} over $\Phi$ is a non-convex quadratic projection problem, i.e., projecting a target point onto the quadratic surface determined by $R$.  
This quadratic projection problem can be efficiently solved by an alternating minimization algorithm when $\|RJ\|_{\rF}$ is small, where $J=({\bf 0}\ I_{(d^2+d)/2})^T\in\mathbb{R}^{(2+3d+d^2)/2\times (d^2+d)/2}$. This motivates us to add a regularizer $\lambda\|RJ\|_{\rF}^2$ to \eqref{equ:6} and solve the regularized QMF problem:
\#
\min_{R,\Phi\atop \Phi\Phi^T=I_d,\Phi{\bf 1}_m={\bf 0}}\ell_{\lambda}(R,\Phi)=\|X-RT(\Phi)\|_{\rF}^2+\lambda\|RJ\|_{\rF}^2.\label{equ:R}
\#
To solve \eqref{equ:R}, an alternating minimization algorithm is proposed in Section \ref{sec:3}. We also discuss how to tune $\lambda$ properly. 
 
Our contributions are four-fold. First, motivated by manifold learning problems, we introduce the quadratic matrix factorization model and propose an alternating minimization algorithm for solving~\eqref{equ:6}.
Second, we establish the theoretical convergence property of the QMF algorithm.
Third, motivated by the theoretical analysis of the QMF algorithm, we propose a regularized QMF algorithm and give an adaptive parameter tuning method.  Finally, we apply the regularized QMF algorithm to solve general manifold learning problems. Numerically, we examine the performance of the proposed method on a synthetic manifold learning dataset and two real-world  datasets, including the MNIST handwritten dataset and a cryogenic electron microscopy dataset, and demonstrate the superiority of the proposed method over its competitors.

\subsection{Related Work and Paper Organization}

Our QMF model is different from the problems studied in \citep{yang2011unified,yang2012quadratic}, which are also referred to as quadratic matrix factorization. They consider approximating a matrix by a product of multiple low-rank matrices, and some factor matrix appears twice in the approximation, that is, $X\approx UU^T$ or $X\approx AUBU^TC$ with $U$ being an unknown factor. They focus on estimating $U$. In contrast, our paper is motivated by manifold learning problems and focuses on solving \eqref{equ:1} with quadratic constraints: some rows of $V$ are quadratic functions of the other rows of $V$. Their algorithms cannot be applied to solve our QMF problem \eqref{equ:6}.

Many manifold learning methods, such as LLE \citep{roweis2000nonlinear}, Isomap \citep{tenenbaum2000global}, Laplacian eigenmaps \citep{belkin2003laplacian} and diffusion maps \citep{coifman2006diffusion}, aim to find a lower-dimensional representation of the dataset while preserving certain geometric structures. In contrast, our target is to recover the underlying manifold structure on which the dataset lies. Most algorithms with the same purpose are based on tangent space estimation \citep{ozertem2011locally,genovese2014nonparametric,fefferman2018fitting,gong2010locally,chen2015asymptotic}. These methods share one common limitation that they do not take the higher-order smoothness into account. Compared with the aforementioned methods, the local polynomial approximation algorithm, which takes higher-order smoothness into account, could achieve a better convergence rate as shown in \citep{aamari2019nonasymptotic}. However, \citep{aamari2019nonasymptotic} and  \citep{aizenbud2021non} mainly study the statistical properties of the local polynomial fitting algorithms, leaving several important computational issues, such as the algorithmic convergence properties untouched. Our paper addresses these computational issues by providing algorithmic convergence properties and a regularized QMF algorithm that potentially avoids over-fitting.

The rest of the paper proceeds as follows. In Section \ref{model}, we describe the alternating minimization algorithm for solving the QMF problem \eqref{equ:6}. As a key component, we propose an alternating minimization algorithm to solve the quadratic projection problem and present its theoretical analysis. We establish the algorithmic convergence property of QMF in Section \ref{sec:4}. In Section \ref{sec:3}, we develop a regularized QMF algorithm and discuss the tuning method.
Applications to manifold learning algorithms are given in Section \ref{application}, and numerical experiments are carried out in Section \ref{experiment}. We conclude this paper with several remarks in Section \ref{sec:7} and leave technical proofs in the Appendix.

\subsection{Notation}\label{sec:notation}

Throughout this paper, we denote by $\|\cdot\|$ and $\|\cdot\|_{\rF}$ the spectral norm and the Frobenius norm respectively. For a vector $v\in\mathbb{R}^D$, we use $\|v\|_1=\sum_{i=1}^D|v_i|$ to denote its $\ell_1$ norm. For a matrix $M\in\mathbb{R}^{r\times m}$, denote by $\sigma_i(M)$, $\sigma_{\min}(M)=\sigma_{\min\{r,m\}}(M)$, and $\|M\|_{2,1}=\sum_{i=1}^r\|M_{i\cdot}\|$ the $i$-th largest singular value, the smallest singular value, and the $\ell_{2,1}$ norm of $M$ respectively. Here $M_{i\cdot}$ is the $i$-th row of $M$. Also, denote by $M^\dagger$  the Moore-Penrose inverse of $M$ and by $P_M=M^T(MM^T)^{\dagger}M\in\mathbb{R}^{m\times m}$ the projection matrix corresponding to the row subspace of $M$, i.e., the subspace in $\mathbb{R}^m$ spanned by the rows of $M$. Let $\mathcal{B}\in\mathbb{R}^{D\times d\times d}$ be a tensor, then $\mathcal{B}(\tau,\eta)\in\mathbb{R}^D$ for any $\tau,\eta\in\mathbb{R}^d$. We use $B_k\in\mathbb{R}^{d\times d}$ to denote the $k$-th slice of $\mathcal{B}$, i.e., $\mathcal{B}(\tau,\eta)_k=\tau^TB_k\eta=\langle B_k,\eta\tau^T\rangle$ for all $\tau,\eta\in\mathbb{R}^d$ and $k=1,\ldots,D$. We refer to $\mathcal{B}$ as a symmetric tensor if $\{B_k\}_{k=1}^D$ are all symmetric. For any $\eta\in\mathbb{R}^d$, define ${\cal B}_\eta$ as the action of ${\cal B}$ on the vector $\eta$:
\begin{equation}
    {\cal B}_\eta=[B_1\eta,\ldots,B_D\eta]^T\in\mathbb{R}^{D\times d},\label{equ:24}
\end{equation}
where $B_k$ denotes the $k$-th slice of $\mathcal{B}$. 
Thus, $\mathcal{B}(\tau,\eta)={\cal B}_{\eta} \tau={\cal B}_\tau \eta$ when $\mathcal{B}$ is symmetric.
Denote by $\mathcal{B}^*(\cdot):\mathbb{R}^D\mapsto\mathbb{R}^{d\times d}$ the adjoint operator of $\mathcal{B}$:
\begin{equation*}
    \mathcal{B}^*(c)=\sum_{k=1}^Dc_kB_k\in\mathbb{R}^{d\times d},\quad\forall c\in\mathbb{R}^D,
\end{equation*}
where $B_k$ is the $k$-th slice of $\mathcal{B}$. Let ${\bf 1}_m=[1,\ldots,1]^T\in\mathbb{R}^m$ and ${\bf 0}$ be an all-zero matrix, whose size depends on the context. In addition, let $I_d$ be the identity matrix of size $d\times d$. Given two matrices $A,D\in\mathbb{R}^{d\times d}$, we use $A\succeq D$ (resp. $A\preceq D$) to indicate that $A-D$ (resp. $D-A$) is a positive semi-definite matrix. 

\section{Quadratic Matrix Factorization}\label{model}

This section presents an alternating minimization algorithm for solving quadratic matrix factorization. Given $\{x_i\}_{i=1}^m\subseteq\mathbb{R}^D$, the goal is to solve 
\begin{equation}\label{equ:5}
\min_{f\in{\cal F}, \{\tau_i\}_{i=1}^m\subseteq{\mathbb R}^d}\sum_{i=1}^m\|x_i-f(\tau_i)\|^2_{\rF},
\end{equation}
where ${\cal F}$ is the quadratic function class \eqref{equ:F}.
Before proceeding to the algorithm, let us first illustrate the advantages of the quadratic function class over the linear function class in a swiss roll fitting example. As in Figure~\ref{fig:1}, we generate noisy data points near a swiss roll, fit such data using linear and quadratic functions, and compare the fitted curves with the underlying truth.
It turns out that linear fitting tends to return a polygon, while quadratic fitting could produce curved lines that recover the underlying truth better. The difference between linear and quadratic fitting becomes more significant in the central region, where the curvature is large.
This coincides with the intuition that quadratic fitting performs better because it takes the curvature into consideration, while linear fitting does not.

\begin{figure}[t] 
   \centering
   \includegraphics[width=\linewidth]{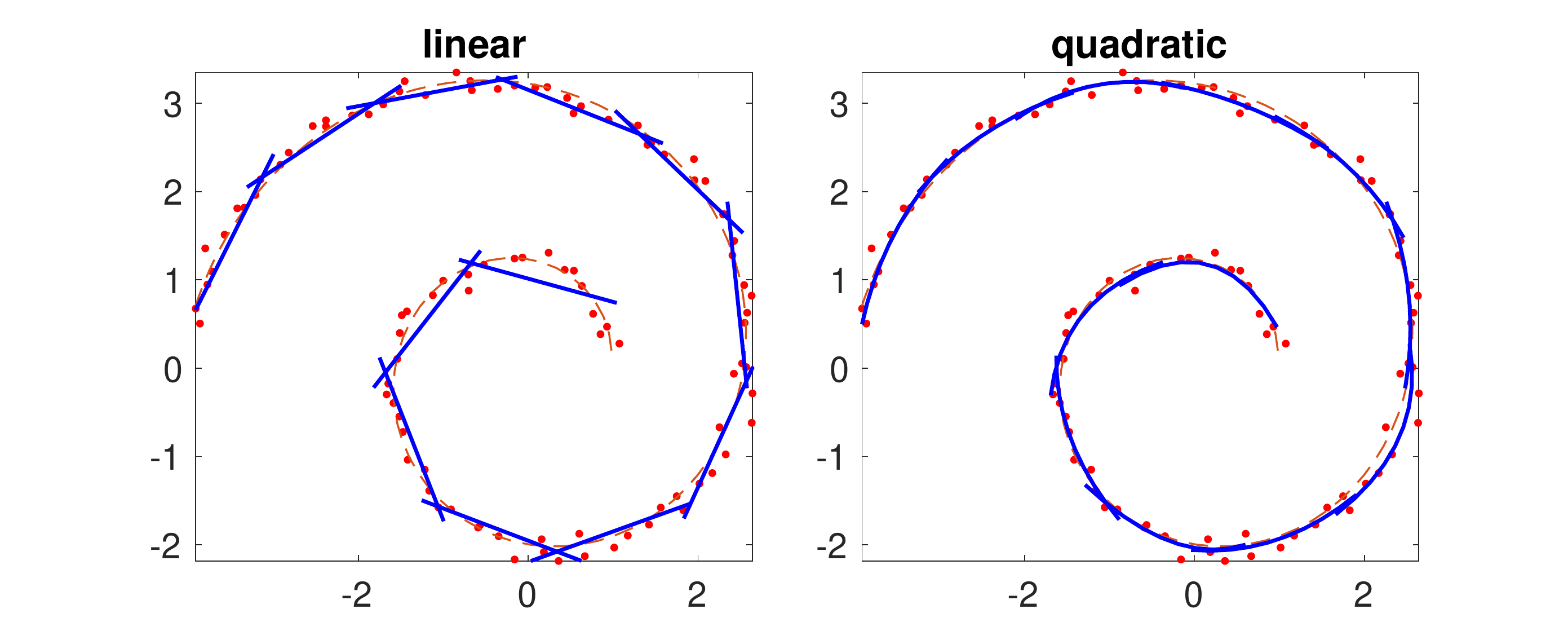} 
   \caption{A comparison between the linear and quadratic fitting in the swiss roll fitting problem. 
   The above figures display the fitted curves using linear and quadratic functions, respectively. The dots represent the raw data and the dashed lines represent the underlying truth.}
   \label{fig:1}
\end{figure}

To formally present the algorithm,
 we need some notations. Recall that $\psi(\cdot)$ and $\xi(\cdot)$ are given by \eqref{xi}.
For any symmetric tensor $\mathcal{B}\in\mathbb{R}^{D\times d\times d}$, 
there exists a unique matrix  $Q\in\mathbb{R}^{D\times (d^2+d)/2}$ such that
\begin{equation}
\mathcal{B}(\tau,\tau)=Q\psi(\tau),\quad\forall \tau\in\mathbb{R}^d.\label{BQ}
\end{equation}
In particular, the $k$-th row $Q_{k\cdot}$ of $Q$ relates to the $k$-th slice of $\mathcal{B}$ via the equality $Q_{k\cdot}\psi(\tau)=\tau^TB_k\tau$.
We shall refer to such $Q$ as the matrix representation of the symmetric tensor $\mathcal{B}$.
Similarly, for any quadratic function $f(\tau)=c+A\tau+{\cal B}(\tau,\tau)$, there exists a unique matrix $R\in\mathbb{R}^{D\times (2+3d+d^2)/2}$ such that $f(\tau)=R\xi(\tau),\forall \tau\in\mathbb{R}^d.$
Indeed, $R=[c, A, Q]$ with $Q$ being the matrix representation of the symmetric tensor $\cal B$.
Let $X=[x_1,\ldots,x_m]\in\mathbb{R}^{D\times m}$, $\Phi=[\tau_1,...,\tau_m]\in {\mathbb R}^{d\times m}$, and 
\begin{equation}\label{phi_psi}
\Psi(\Phi) = [\psi(\tau_1),...,\psi(\tau_m)]\in{\mathbb R}^{\frac{d^2+d}{2}\times m}.
\end{equation}
Then the optimization problem~\eqref{equ:5} is equivalent to 
\begin{align}\label{equ:9}
\min_{R,\Phi}\ell(R,\Phi)=\ell(c,A,Q,\Phi) = \|X - RT(\Phi)
\|_{\rF}^2,
\end{align}
where 
\begin{equation}\label{equ:10}
\begin{aligned}
&R=R(c,A,Q)=[c,A,Q],\\
&T(\Phi)=[\xi(\tau_1),\ldots,\xi(\tau_m)]
=\left[{\bf 1}_m, \Phi^T, \Psi^T(\Phi)
\right]^T.
\end{aligned}
\end{equation}
This can be viewed as a matrix factorization problem, with constraints given by \eqref{equ:10}. 
The last $(d^2+d)/2$ rows of $T(\Phi)$ in the constraints \eqref{equ:10} are determined by the second-to-$(1+d)$th rows of $T(\Phi)$ through the mapping $\psi$.
Recall that $\psi(\tau)$ collects all quadratic and interaction terms of $\{\tau_{[i]}\}_{i=1}^d$, thus these constraints specify all possible quadratic constraints. Problem \eqref{equ:9} is thus referred to as QMF. LMF is a special case of QMF with additional constraints that $Q=0$.

However,  \eqref{equ:9} suffers from non-identifiability issues. 
This is because the minima of \eqref{equ:9} are determined by $T(\Phi)$ only through its row space.  
To see this, we fix $\Phi$ and consider minimizing $\ell(R,\Phi)$ with respect to $R$ only. The minimizer $\widetilde R$ and the product $\widetilde RT(\Phi)$ are given by 
\begin{equation}\label{equ:13}
\begin{aligned}
   & \widetilde R=\argmin_{R}\ell(R,\Phi)=XT(\Phi)^T(T(\Phi)T(\Phi)^T)^{\dagger},\\
   & \widetilde RT(\Phi)=XT(\Phi)^T(T(\Phi)T(\Phi)^T)^{\dagger}T(\Phi)=XP_{T(\Phi)},
\end{aligned}
\end{equation}
where $M^\dagger$ denotes the Moore–Penrose inverse of $M$ and $P_{T(\Phi)}=T(\Phi)^T(T(\Phi)T(\Phi)^T)^{\dagger}T(\Phi)\in\mathbb{R}^{m\times m}$.  
Thus the loss of $\widetilde RT(\Phi)$ only depends on the row subspace of $T(\Phi)$.
Substituting \eqref{equ:13} into \eqref{equ:9}, we obtain
\begin{align*}
\min_{\Phi}\min_R\ell(R,\Phi)=\min_{\Phi}\|X-XP_{T(\Phi)}\|_{\rF}^2.
\end{align*}
In particular, if $\Phi_1$ and $\Phi_2$ satisfy $P_{T(\Phi_1)}=P_{T(\Phi_2)}$, then $\min_R\ell(R,\Phi_1)=\min_R\ell(R,\Phi_2)$ and thus it is impossible to distinguish $\Phi_1$ and $\Phi_2$ when optimizing $\eqref{equ:9}$. Proposition~\ref{prop:1} provides concrete transformations on $\Phi$ such that $\min_R\ell(R,\Phi)$ stays the same.

\begin{proposition}\label{prop:1}
Suppose $\Phi'=Z\Phi+u{\bf 1}_m^T$ for some invertible matrix $Z\in\mathbb{R}^{d\times d}$ and some vector $u\in\mathbb{R}^d$. Then for any $R$, there exists $R'$ such that
$\ell(R',\Phi') = \ell(R,\Phi)$. In particular, we have $\min_R\ell(R,\Phi')=\min_{R}\ell(R,\Phi)$.
\end{proposition}

The proof of Proposition~\ref{prop:1} is left in the Appendix. To overcome non-identifiability issues, we add additional  constraints $\Phi\Phi^T=I_d$ and $\Phi{\bf 1}_m=0$ to \eqref{equ:9} and solve
\begin{equation}\label{equ:16}
\min_{R,\Phi \atop \Phi\Phi^T = I_d,\Phi {\bf 1}_m=0}\ell(R,\Phi) = \|X - R
T(\Phi)
\|_{\rF}^2,
\end{equation}
where 
$R=R(c,A,Q)$ and $T(\Phi)$ are defined in \eqref{equ:10}. Several remarks follow. First, by Proposition \ref{prop:1}, the minima of the constrained problem \eqref{equ:16} and the unconstrained problem \eqref{equ:9} are the same. Thus these two problems are equivalent. Second, by introducing constraints $\Phi\Phi^T=I_d$, we reduce the solution space from $\mathbb{R}^{d\times m}$ to the Stiefel manifold, which potentially helps speed up the optimization procedure \citep{edelman1998geometry}.
Third, we still do not have any restrictions on $R$ in \eqref{equ:16}, so optimizing \eqref{equ:16} over $R$ with fixed $\Phi$ is still a regression problem with its solution given by~\eqref{equ:13}.
Fourth, the constraint $\Phi\Phi^T=I_d$ enforces the scale of $\Phi$ to be neither too large nor too small. Thus, the configuration of the approximation $RT(\Phi)$ largely depends on $R$ and can be controlled by regularizing $R$ properly. We will explore this last point in Section \ref{sec:3}. 

\begin{algorithm}[t]
\caption{An alternating minimization algorithm for quadratic matrix factorization.}
\label{alg:1}
\KwData{$X=[x_1,\ldots,x_m]\in\mathbb{R}^{D\times m}$}
\KwResult{$\Phi=[\tau_1,\ldots,\tau_m]\in\mathbb{R}^{d\times m}$, quadratic function $f(\tau)= R_t\xi(\tau)$.}

{
initialize $\Phi_0\in\mathbb{R}^{d\times m}$ as the top $d$ eigenvectors of $G=(X-\bar x{\bf 1}_m^T)^T(X-\bar x{\bf 1}_m^T)$\;
\While{$\|\Phi_{t}^T \Phi_{t}-\Phi_{t-1}^T\Phi_{t-1}\|>\epsilon$}
{update $R_{t} = \argmin_R \ell(R, \Phi_{t-1})$ as in \eqref{equ:13}\;
{\For{$i=1$ \KwTo $m$}{solve the $i$-th projection problem $\widetilde\tau_{i,t}=\argmin_{\tau\in\mathbb{R}^{d}}\|x_i-R_t\xi(\tau)\|^2$\;}}
set $\widetilde\Phi_t=[\widetilde\tau_{1,t},\ldots,\widetilde\tau_{m,t}]$\;
update $\Phi_t=Z_t\widetilde\Phi_t(I_m-{\bf 1}_m{\bf 1}_m^T/m)$ with $Z_t$ given by \eqref{equ:19}\;
}
}
\end{algorithm}

Now we present our first main algorithm. We adopt an alternating minimization strategy to solve \eqref{equ:16}. To begin with, we initialize $\Phi_0\in\mathbb{R}^{d\times m}$ as the top $d$ eigenvectors of the gram matrix $G=(X-\bar x{\bf 1}_m^T)^T(X-\bar x{\bf 1}_m^T)$, where $\bar x=\frac{1}{m}\sum_{i=1}^mx_i$. During the $t$-th loop, we first fix $\Phi=\Phi_{t-1}$ and update $R_{t} = \argmin_R \ell(R, \Phi_{t-1})$ as in \eqref{equ:13}. Next, we fix $R=R_t$ and update $\widetilde\Phi_t=\argmin_{\Phi}\ell(R_t,\Phi)$. 
This is a separable problem in the sense that $\widetilde \Phi_t$  is given by  $\widetilde\Phi_t=[\widetilde\tau_{1,t},\ldots,\widetilde\tau_{m,t}]$ with
\begin{equation}
\widetilde\tau_{i,t}=\argmin_{\tau\in\mathbb{R}^d}\|x_i-R_t\xi(\tau)\|^2,\label{equ:18}
\end{equation}
where $\xi(\tau)$ is defined in \eqref{xi}.
We refer to \eqref{equ:18} as a quadratic projection problem because 
it finds the closest point $R_t\xi(\widetilde \tau_{i,t})$ to $x_i$ on the quadratic surface $f_t(\tau)=R_t\xi(\tau)$.
At the end of the $t$-th loop, we set $\Phi_t=Z_t\widetilde\Phi_t(I_m-{\bf 1}_m{\bf 1}_m^T/m)$ with $Z_t$ given by
\begin{equation}
    Z_t=(\widetilde\Phi_t(I_m-{\bf 1}_m^T{\bf 1}_m/m)\widetilde\Phi_t^T)^{-1/2}.\label{equ:19}
\end{equation}
This ensures that  the constraints $\Phi_t\Phi_t^T=I_d$ and $\Phi_t{\bf 1}_m=0$ hold. Given a precision level $\epsilon>0$, we will repeat the above iterations until the stopping criteria $\|\Phi_t^T\Phi_t-\Phi_{t-1}^T\Phi_{t-1}\|\leq\epsilon$ is met, i.e., the row subspace of $\Phi_t$ stabilizes.
Algorithm~\ref{alg:1} summarizes the details. Till now, the only issue we have not yet addressed is how to solve \eqref{equ:18}, which will be discussed in the following subsection.

\subsection{Quadratic Projection}\label{sec:2.2}
Since the parameters $R_t$ and $x_i$ are fixed when solving \eqref{equ:18},
we omit the subscripts $i$ and $t$ for simplicity. We rewrite the loss function in \eqref{equ:18} as
\begin{equation}\label{h}
    h(\tau)=\|x-c-A\tau-\mathcal{B}(\tau,\tau)\|^2,
\end{equation}
where $c,A,{\cal B}$ are determined by $R$ via \eqref{BQ} and \eqref{equ:10}.
Minimizing \eqref{h} with respect to $\tau$ is non-convex,
and we consider minimizing the following  surrogate loss instead:
\begin{gather*}
\min_{\tau,\eta}g(\tau, \eta)=  \frac{1}{2}\|x-c-A\tau -{\cal B}(\tau,\eta)\|^2 \\ + \frac{1}{2}\|x-c-A\eta -{\cal B}(\tau,\eta)\|^2.
\end{gather*}
Note that $g$ is symmetric in $\tau$ and $\eta$, and $g(\tau,\tau)=h(\tau)$.
Starting from $\tau_0=\eta_0={\bf 0}\in\mathbb{R}^d$, we update $\{\tau_s,\eta_s\}$ iteratively in the following manner:
\begin{equation}
\left\{
\begin{array}{l}
\tau_{s} = \argmin_{\tau} g(\tau, \eta_{s-1}),\\
\eta_{s} = \argmin_{\eta} g(\tau_{s}, \eta). 
\end{array}
\right. \label{equ:21}
\end{equation}
Upon convergence such that $(\tau^*,\eta^*)=\argmin_{\tau,\eta}g(\tau,\eta)$, we take $\tau^*$ to be the solution to \eqref{equ:18}.
In what follows, we show how \eqref{equ:21} can be efficiently solved  and prove that $(\tau_s,\eta_s)$ converges to $(\tau^*,\tau^*)$ for some stationary point $\tau^*$ of $h$ under certain conditions.


The update rule \eqref{equ:21} can be implemented efficiently as all iterates admit closed-form solutions. Let us fix $\eta=\eta_{s-1}$ and consider optimizing $g(\tau,\eta_{s-1})$ over $\tau$. 
Recall ${\cal B}_{\eta}$ is given by \eqref{equ:24} for any $\eta\in\mathbb{R}^{d}$
and  $\mathcal{B}(\tau,\eta)={\cal B}_\tau \eta={\cal B}_\eta \tau$. Then $g(\tau,\eta_{s-1})$ can be rewritten as
\begin{gather*}
    g(\tau,\eta_{s-1})=\frac{1}{2}\|x-c-(A+{\cal B}_{\eta_{s-1}})\tau\|^2\\
    +\frac{1}{2}\|x-c-A\eta_{s-1}-{\cal B}_{\eta_{s-1}}\tau\|^2.
\end{gather*}
This is a quadratic function of $\tau$, so $\tau_s=\argmin_{\tau}g(\tau,\eta_{s-1})$ admits a closed-form solution
\begin{equation*}
    \tau_s=\Gamma_{\eta_{s-1}}^{-1}\zeta_{\eta_{s-1}},
\end{equation*}
where $\Gamma_{\eta_{s-1}}$ and $\zeta_{\eta_{s-1}}$ are
\begin{align}
    \Gamma_{\eta_{s-1}}&=(A+{\cal B}_{\eta_{s-1}})^T(A+{\cal B}_{\eta_{s-1}})+{\cal B}_{\eta_{s-1}}^T{\cal B}_{\eta_{s-1}},\label{equ:25}\\
    \zeta_{\eta_{s-1}}&=(A+{\cal B}_{\eta_{s-1}})^T(x-c)+{\cal B}_{\eta_{s-1}}^T(x-c-A\eta_{s-1}).\label{equ:26}
\end{align}
Since $g(\tau,\eta)$ is symmetric in $\tau$ and $\eta$, the dual problem $\eta_s=\argmin_{\eta}g(\tau_s,\eta)$ can be solved similarly as
\begin{equation*}
    \eta_s=\Gamma_{\tau_s}^{-1}\zeta_{\tau_s},
\end{equation*}
where $\Gamma_{\tau_s}$ and $\zeta_{\tau_s}$ are given by \eqref{equ:25} and \eqref{equ:26} with $\eta_{s-1}$ replaced by $\tau_s$.


Now we provide conditions under which $(\tau_s,\eta_s)$ converges such that $\tau^*=\lim_{s}\tau_s=\lim_{s}\eta_s$, and $(\tau^*,\tau^*)$ is a first-order stationary point of $g$.  This implies that $\tau^*$ is a stationary point of $h$.
The key is to analyze the Hessian matrix of $g$:
\begin{equation*}
    H_g(\tau,\eta)=
    \begin{bmatrix}
    \Gamma_\eta & H_{\eta\tau}\\
    H_{\tau\eta} & \Gamma_\tau
    \end{bmatrix},
\end{equation*}
where $\Gamma_\eta$  and $\Gamma_\tau$ are given by \eqref{equ:25} with parameter $\eta$ and $\tau$ respectively and $H_{\eta\tau}$ and $H_{\tau\eta}$ are given by 
\begin{equation}
\begin{aligned}
    H_{\eta\tau}=H_{\eta\tau}^T= &-2\mathcal{B}^*(x-c-A(\tau+\eta)/2-\mathcal{B}(\tau,\eta))
    \\
    &+{\cal B}_\eta^T(A+{\cal B}_\tau)+(A+{\cal B}_\eta)^T {\cal B}_\tau.\label{equ:32}
    \end{aligned}
\end{equation}
Here $B_\eta$ and $B_\tau$ are defined in \eqref{equ:24}, and $\mathcal{B}^*(\cdot)$ represents the adjoint operator of $\mathcal{B}$. 
$\Gamma_\eta$ is always positive definite when $\sigma_{d}(A)>0$ because
\begin{equation}
    \Gamma_\eta=\frac{1}{2}A^TA+(\frac{1}{\sqrt{2}}A+\sqrt{2}{\cal B}_\eta)^T(\frac{1}{\sqrt{2}}A+\sqrt{2}{\cal B}_\eta)\succeq \frac{1}{2}A^TA.\label{equ:34}
\end{equation}
The following theorem shows that if $(\tau_s,\eta_s)$ is bounded, then under certain conditions on $\mathcal{B}$, we have $(\tau_s,\eta_s)$ converges, $\tau^*=\lim_{s\to\infty}\tau_s=\lim_{s\to\infty}\eta_s$, and $(\tau^*,\tau^*)$ is a stationary point of $g$.

\begin{theorem}\label{thm:2}
    Suppose $\sigma_{d}(A)>0$ and define $ \mathcal{S}_{\alpha}=\{\tau\in\mathbb{R}^d\mid\|\tau\|\leq\alpha\}$ for some $\alpha>0$. Denote by $B_k$ the $k$-th slice of $\mathcal{B}$ for $k=1,\ldots,D$. If $\mathfrak{b}=\max_{k}\sigma_1(B_k)$ satisfies 
    \#
    (2\|x-c\|_1+4\alpha\|A\|_{2,1})\mathfrak{b}+3D\alpha^2\mathfrak{b}^2\leq \sigma^2_d(A)/4,\label{equ:34b}
    \#
    then $H_g(\tau,\eta)$ is positive definite and $\sigma_{\min}(H_g(\tau,\eta))\geq\sigma_d^2(A)/4$  for all $\tau,\eta\in\mathcal{S}_\alpha$. 
    If the sequence $(\tau_s,\eta_s)$ obtained by \eqref{equ:21} falls into the region $\mathcal{S}_\alpha\times \mathcal{S}_\alpha$ for sufficiently large $s$, then $(\tau_s,\eta_s)$ converges, $\tau^*=\lim_{s\to\infty}\tau_s=\lim_{s\to\infty}\eta_s$, and $(\tau^*,\tau^*)$ is a stationary point of $g$.
\end{theorem}
\begin{proof}
    Since $\sigma_{d}(A)>0$, by \eqref{equ:34}, $\Gamma_\eta$ is positive definite with $\sigma_{\min}(\Gamma_\eta)\geq\sigma_{d}^2(A)/2$. By symmetry, such property holds for $\Gamma_\tau$ as well.
    As for $H_{\eta\tau}$ and $H_{\tau\eta}$, 
    we can use \eqref{equ:34b} to show that
    \begin{equation}     \sigma_1(H_{\eta\tau})=\sigma_1(H_{\tau\eta})\leq\sigma_{d}^2(A)/4,\label{equ:35}
    \end{equation}
    holds for any $\tau,\eta\in\mathcal{S}_\alpha$. In particular, for any $\tau,\eta\in{\cal S}_{\alpha}$, we have
   \[
    \begin{array}{ll}
    \sigma_1({\cal B}^*({\cal B}(\tau,\eta)))\leq  D\alpha^2 \mathfrak{b}^2,
    &\sigma_1({\cal B}_{\eta}^TA)\leq  \alpha\|A\|_{2,1}\mathfrak{b},\\
    \sigma_1(A^T{\cal B}_{\tau})\leq \alpha\|A\|_{2,1}\mathfrak{b}, 
    &\sigma_1({\cal B}_{\eta}^T{\cal B}_{\tau})\leq  D\alpha^2\mathfrak{b}^2,
    \end{array}
    \]
    Similarly,
    \[
    \begin{aligned}
    &\sigma_1({\cal B}^*(x-c-A(\tau+\eta)/2)) \\
    \leq & (\|x-c\|_1+\|A(\tau+\eta)/2\|_1)\mathfrak{b} \leq  (\|x-c\|_1+\alpha\|A\|_{2,1})\mathfrak{b}.
    \end{aligned}
    \]
    Combining these inequalities with  \eqref{equ:32}, we obtain
    \$
    \begin{aligned}
    \sigma_1(H_{\eta\tau})&=\sigma_1(H_{\tau\eta})\\
    &\leq (2\|x-c\|_1+4\alpha\|A\|_{2,1})\mathfrak{b}+3D\alpha^2\mathfrak{b}^2.
    \end{aligned}
    \$
    Then \eqref{equ:35} follows from the condition \eqref{equ:34b}.
    To proceed, we decompose $H_g(\tau,\eta)$ into the following summation of a positive definite matrix and a symmetric matrix:
    \begin{equation}
        H_g(\tau,\eta)=
        \begin{bmatrix}
        \Gamma_\eta & {\bf 0}\\
        {\bf 0} & \Gamma_\tau
        \end{bmatrix}+
        \begin{bmatrix}
        {\bf 0} & H_{\eta\tau}\\
        H_{\tau\eta} & {\bf 0}
        \end{bmatrix}.\label{equ:36}
    \end{equation}
    For the second term, it follows from \eqref{equ:35} that
    \begin{equation*}
        \sigma_1\left(\begin{bmatrix}
        {\bf 0} & H_{\eta\tau}\\
        H_{\tau\eta} & {\bf 0}
        \end{bmatrix}\right)= \sigma_1(H_{\eta\tau})\leq\sigma_d^2(A)/4,
    \end{equation*}
    for any $\tau,\eta\in\mathcal{S}_\alpha$. Since  $\min\{\sigma_{\min}(\Gamma_\tau),\sigma_{\min}(\Gamma_{\eta})\}\geq\sigma_d^2(A)/2$, the first term in \eqref{equ:36} is positive definite, $H_g(\tau,\eta)$ is positive definite, and
    \begin{equation}
    \begin{aligned}
        \sigma_{\min}(H_g(\tau,\eta))\geq&\min\{\sigma_{\min}(\Gamma_\tau),\sigma_{\min}(\Gamma_{\eta}\}-\sigma_1(H_{\eta\tau})\\
        \geq&\sigma_d^2(A)/4, \qquad \forall \tau,\eta\in\mathcal{S}_\alpha.\label{equ:38}
        \end{aligned}
    \end{equation}
    Since the region $\mathcal{S}_\alpha\times \mathcal{S}_\alpha$ is convex, \eqref{equ:38} implies the strong convexity of $g$ over $\mathcal{S}_\alpha\times \mathcal{S}_\alpha$.
    
Suppose that the sequence $(\tau_s,\eta_s)$ generated by \eqref{equ:21} falls into the region $\mathcal{S}_\alpha\times \mathcal{S}_\alpha$ for sufficiently large $s$. Since $g$ is smooth and strongly convex over the region $\mathcal{S}_\alpha\times \mathcal{S}_\alpha$, it is well-known that $(\tau_s,\eta_s)$ would converge to the unique first-order stationary point of $g$ in $\mathcal{S}_\alpha\times \mathcal{S}_\alpha$ \citep{beck2017first}. Next, we will prove $\lim_{s\to\infty}\tau_s=\lim_{s\to\infty}\eta_s$ by contradiction. 
    In specific, define $\tau^*=\lim_{s\to\infty}\tau_s$ and $\eta^*=\lim_{s\to\infty}\eta_s$ and assume $\tau^*\neq\eta^*$.
    Since $(\tau^*,\eta^*)$ is a first-order stationary point of $g$ in $\mathcal{S}_{\alpha}\times\mathcal{S}_{\alpha}$, by symmetry, we know $(\eta^*,\tau^*)$ is also a first-order stationary point of $g$ in $\mathcal{S}_{\alpha}\times\mathcal{S}_{\alpha}$.
    However, this contradicts the uniqueness of the first-order stationary point of $g$ in  $\mathcal{S}_{\alpha}\times\mathcal{S}_{\alpha}$.
\end{proof}
\begin{figure}[t] 
   \centering
   \includegraphics[width=\linewidth]{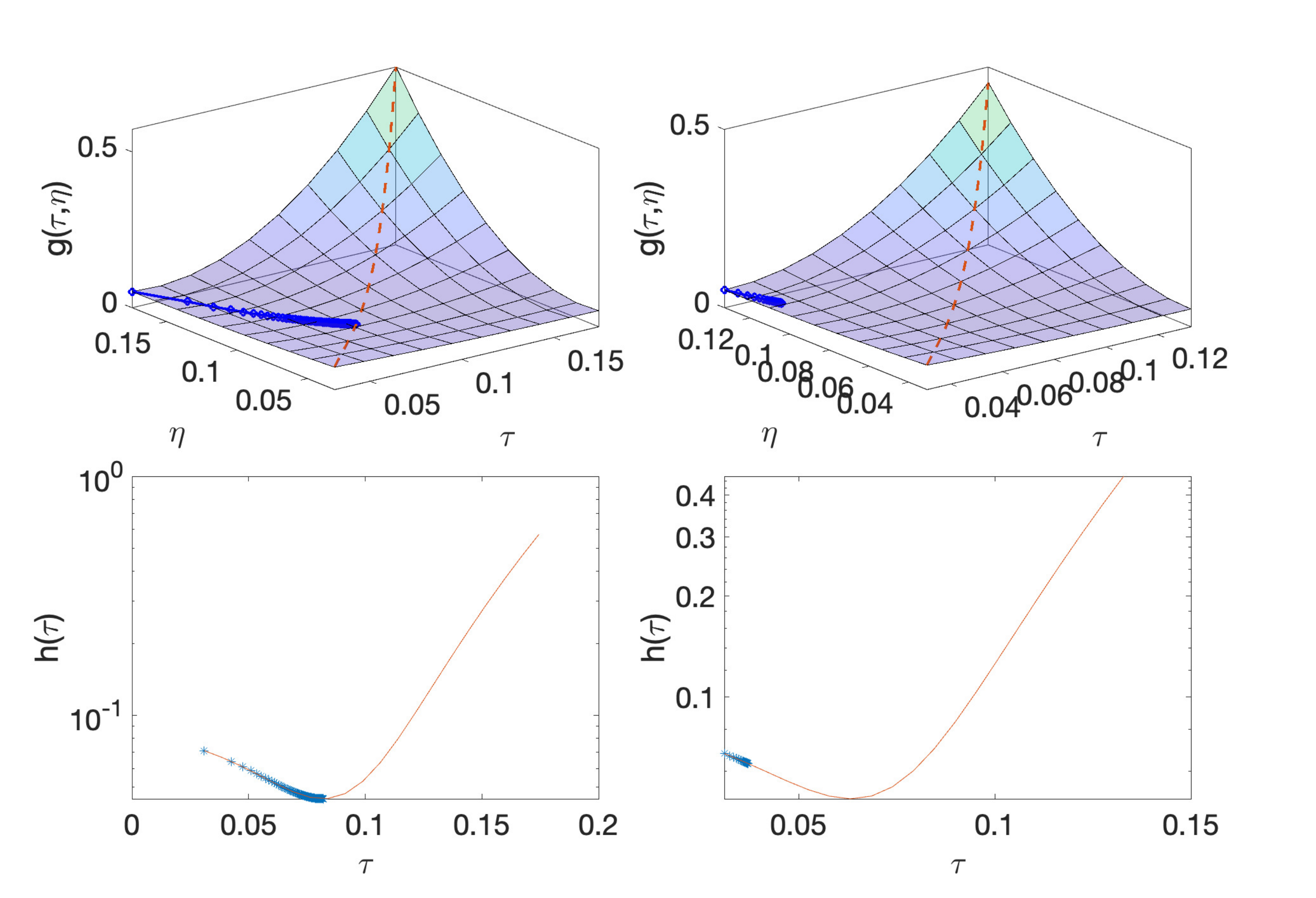} 
   \caption{
   {
   An illustration of the effect of $\mathfrak{b}$ on the 
   convergence 
   behaviors 
   of 
   $\{(\tau_s,\eta_s)\}_{s=1}^{\infty}$ generated by \eqref{equ:21}.} The left column corresponds to the setting $(x,c,A,20{\cal B})$ and the right column corresponds to the setting $(x,c,A,30{\cal B})$, where the parameters are set as follows: $A = [ -0.8979, 1.0086, -0.5422]^T,{\cal B} = [0.7817, -1.4908, -0.3679], c =[ 0.4171, 0.9176, 0.1759]^T$, and $x = [0.2561, 0.7500, 0.0099]^T$.
   {
   The first row displays $g(\tau_s,\eta_s)$ with the background surface representing the function $g$. The dotted line stands for the case $\tau=\eta$.
   The second row displays $h(\tau_s)$ under the two settings.
   }
   }
   \label{fig:D1}
 \end{figure}

Suppose $(\tau_s,\eta_s)\subseteq{\cal S}_{\alpha}\times {\cal S}_{\alpha}$ for sufficiently large $s$. Theorem~\ref{thm:2} proves the convergence of $(\tau_s,\eta_s)$ by establishing the strong convexity of $g$ over ${\cal S}_{\alpha}\times {\cal S}_{\alpha}$ under condition \eqref{equ:34b}.
Condition \eqref{equ:34b} holds when {both} $\alpha$ and $\|x-c\|_1$ are small or  $\mathfrak{b}=\max_k\sigma_1(B_k)$ is small. In particular, condition \eqref{equ:34b} holds if $\mathfrak{b}\leq\mathfrak{b}_0$, where
\begin{align*}
\mathfrak{b}_0=&\frac{-\|x-c\|_1-2\alpha\|A\|_{2,1}}{3D\alpha^2}\\
&+\frac{\sqrt{(\|x-c\|_1+2\alpha\|A\|_{2,1})^2+3D\alpha^2\sigma_d^2(A)/4}}{3D\alpha^2}.
\end{align*}
It trivially holds {when}
$\mathfrak{b}=0$, which corresponds to LMF. In general, by the relationship \eqref{BQ} between $\mathcal{B}$ and $Q$, we have $\sigma_1(B_k)\leq\|B_k\|_{\rF}\leq\|Q_{k\cdot}\|$, where $Q_{k\cdot}$ is the $k$-th row of $Q$. Thus, $\mathfrak{b}$ is small if $\|Q\|_{\rF}$ is small. 

{
To conclude this subsection,
we use an example with $D=3,d=1$ to illustrate the effect of $\mathfrak{b}$ on the convergence of $\{(\tau_s,\eta_s)\}_{s=1}^{\infty}$.
We consider two different settings of $g$ and display the sequence $\{(\tau_s,\eta_s)\}_{s=1}^{\infty}$ in Figure \ref{fig:D1}. The left panel has a smaller $\mathfrak{b}$ than the right panel. The first row of Figure \ref{fig:D1} shows that $(\tau_s,\eta_s)$ converges with $\lim\tau_s=\lim\eta_s$ in the left panel while $(\tau_s,\eta_s)$ converges with $\lim\tau_s\neq\lim\eta_s$ in the right panel. In addition, the second row shows that $\tau_s$ in the left panel converges to the global minimum of $h$  while $\tau_s$ in the right panel does not.
}
{ 
The differences between these two cases lie in the different values of $\mathfrak{b}$. In particular, if $\mathfrak{b}$ is too large as in the right panel,  $(\tau_s,\eta_s)$ would not converge to the same limiting point,  that is $\lim \tau_s\neq\lim \eta_s$. In this case,  $\eta_s$ is not guaranteed to converge to the global minimum of $h$ in general.}


\section{Convergence Properties}\label{sec:4}

This section establishes the convergence property for Algorithm \ref{alg:1}. In particular, we provide conditions under which a limit point of the update sequence is a stationary point of \eqref{equ:9}. 

\begin{theorem}\label{thm:5}
Denote by $\{(R_t,\widetilde\Phi_t,\Phi_t)\}_{t=1}^{\infty}$ the update sequence generated by Algorithm \ref{alg:1}. For any sub-sequence $\{t_j\}_{j=1}^{\infty}$ such that $\lim_{j}\Phi_{t_j}=\Phi^*$, $\lim_{j}\Phi_{t_{j}-1}=\Phi^{**}$, $\lim_{j}\widetilde\Phi_{t_{j}}=\widetilde\Phi^{**}$, if $\sigma_{\min}(T(\Phi^*)T(\Phi^*)^T)>0$ and $\sigma_{\min}(T(\Phi^{**})T(\Phi^{**})^T)>0$, then $R_{t_j+1}$ and $R_{t_j}$ also converge. Denote by $R^*=\lim_j R_{t_j+1}$ and $R^{**}=\lim_jR_{t_j}$. Also, we have $\widetilde\Phi^{**}=\argmin_{\Phi}\ell(R^{**},\Phi)$. If we assume
\begin{equation}
    \ell(R^{**},\Phi^{**})-\ell(R^{**},\widetilde\Phi^{**})\geq\gamma\|\Phi^{**}-\widetilde\Phi^{**}\|_{\rF}^2\label{equ:50}
\end{equation}
holds for some constant $\gamma>0$,
then $\Phi^{**}=\widetilde\Phi^{**}=\Phi^{*}$ and $R^{**}=R^*$ and the accumulation point $(R^*,\Phi^*)$ satisfies the first order Karush-Kuhn-Tucker (KKT) condition of the problem \eqref{equ:9}.
\end{theorem}
\begin{proof}
By Algorithm \ref{alg:1}, the update sequence $\{(R_{t},\widetilde\Phi_t,\Phi_{t})\}_{t=1}^{\infty}$ satisfies
\begin{equation*}\label{equ:51}
\begin{cases}
R_{t} = \argmin_R \ell(R, \Phi_{t-1}),\\
\widetilde{\Phi}_{t} = \argmin_\Phi \ell(R_{t},\Phi),\\
\Phi_{t}=Z_{t}\widetilde{\Phi}_{t}(I_m-{\bf 1}_m{\bf 1}_m^T/m),
\end{cases}
\end{equation*}
where $Z_{t}$ is defined by \eqref{equ:19}. It implies that $\ell(R_{t+1},\Phi_t)$ is monotone decreasing:
\begin{equation}\label{equ:52}
\begin{aligned}
    \ell(R_{t+1},{\Phi}_{t}) =&\min_{R}\ell(R,\Phi_{t})=\min_{R}\ell(R,\widetilde \Phi_t)\\
    \leq& \ell(R_{t},\widetilde{\Phi}_{t})\leq \ell(R_{t},\Phi_{t-1}),
\end{aligned}
\end{equation}
where the second equality follows from Proposition \ref{prop:1}. By the monotone convergence theorem and the fact that $\ell(R_{t+1},\Phi_t)\geq 0$, both $\ell(R_{t},\Phi_{t-1})$ and $\ell(R_{t},\widetilde\Phi_t)$ converge to the same value.

Consider the sub-sequence $\{t_{j}\}_{j=1}^{\infty}$ such that $\lim_j\Phi_{t_{j}}=\Phi^*$ and  $\sigma_{\min}(T(\Phi^*)T(\Phi^*)^T)=\gamma_1>0$. Then $\sigma_{\min}(T(\Phi)T(\Phi)^T)\geq\gamma_1/2>0$ for any $\Phi\in\mathcal{S}^*_{\epsilon_0}=\{\Phi\mid \|\Phi-\Phi^*\|_{\rF}\leq \epsilon_0\}$ with some sufficiently small $\epsilon_0$. Thus, $XT(\Phi)^T(T(\Phi)T(\Phi)^T)^{-1}$ is well-defined for $\Phi\in\mathcal{S}^*_{\epsilon_0}$ and is a continuous function of $\Phi$ over $\mathcal{S}^*_{\epsilon_0}$. Since $\lim_j\Phi_{t_j}=\Phi^*$, we know $\Phi_{t_j}\in\mathcal{S}^*_{\epsilon_0}$ for sufficiently large $j$. By continuity, we have 
\[
\begin{aligned}
    \lim_j R_{t_j+1}=& \lim_j  XT (\Phi_{t_j})^T(T(\Phi_{t_j})T(\Phi_{t_j})^T)^{\dagger} \\
                            =& XT(\Phi^{*})^T(T(\Phi^*)T(\Phi^*)^T)^{-1}.\label{equ:R*}
\end{aligned}
\]
where we use the definition \eqref{equ:13} of $R_{t_j+1}$. Denote by $R^*=\lim_jR_{t_j+1}$. 

Let us further assume that the sub-sequence $\{t_j\}_{j=1}^{\infty}$  satisfies  $\lim_{j}\Phi_{t_{j}-1}=\Phi^{**}$, $\lim_{j}\widetilde\Phi_{t_{j}}=\widetilde\Phi^{**}$, and $\sigma_{\min}(T(\Phi^{**})T(\Phi^{**})^T)>0$. Then $R_{t_j}$ converges by the same argument and  denote by $R^{**}=\lim_jR_{t_j}$. Moreover, we claim that
$\widetilde\Phi^{**}=\argmin_{\Phi}\ell(R^{**},\Phi)$. Otherwise, there exists $\Phi'$ such that $\ell(R^{**},\Phi')<\ell(R^{**},\widetilde\Phi^{**})$. By continuity, there exists some $j_0$ such that $\ell(R_{t_{j_0}},\Phi')<\ell(R^{**},\widetilde \Phi^{**})$, which contradicts to the relationship \eqref{equ:52}.
Furthermore, we assume
\eqref{equ:50} holds for some constant $\gamma>0$.

Now we are in a position to prove that $\Phi^{**}=\widetilde\Phi^{**}=\Phi^*$, $R^{**}=R^*$, and $(R^*,\Phi^*)$ satisfies the KKT condition. By continuity, we have
\begin{equation*}
\begin{aligned}
    \ell(R^{**},\Phi^{**})=&\lim_{j}\ell(R_{t_j},\Phi_{t_j-1})\\
    =&\lim_{j}\ell(R_{t_j},\widetilde\Phi_{t_j})=\ell(R^{**},\widetilde\Phi^{**}),
\end{aligned}
\end{equation*}
where the second equality follows from \eqref{equ:52}. By \eqref{equ:50}, we have $\Phi^{**}=\widetilde\Phi^{**}$, or equivalently $\|\Phi_{t_{j}-1}-\widetilde\Phi_{t_{j}}\|_{\rF}^2\to 0$ as $j\to\infty$. 
Since $\Phi_{t_{j}}=\argmin_{\Phi\Phi^T=I_d,\Phi{\bf 1}_m=0}\|\Phi-\widetilde\Phi_{t_{j}}\|_{\rF}^2$, we have $\|\Phi_{t_{j}}-\widetilde\Phi_{t_{j}}\|_{\rF}^2\leq\|\Phi_{t_{j}-1}-\widetilde\Phi_{t_{j}}\|_{\rF}^2$ and thus $\|\Phi_{t_{j}}-\widetilde\Phi_{t_{j}}\|_{\rF}^2\to0$ as $l\to\infty$. Then the fact $\Phi_{t_{j}}\to\Phi^*$ implies that $\Phi^{**}=\widetilde\Phi^{**}=\Phi^*$. By \eqref{equ:R*} and its $R^{**}$ version, we have $R^{**}=R^*$. Finally, by taking the limit on the following optimality conditions: 
\begin{equation*}
\begin{cases}
\nabla_\Phi \ell(R_{t_{j}},\Phi)|_{\Phi=\widetilde{\Phi}_{t_{j}}} =0,\\
\Phi_{t_{j}}=Z_{t_{j}}\widetilde{\Phi}_{t_{j}}(I_m-{\bf 1}_m{\bf 1}_m^T/m),\\
\nabla_R \ell(R, \Phi_{t_{j}})|_{R = R_{t_{j}+1}} =0,
\end{cases}
\end{equation*}
we prove that $\{R^*,\Phi^*\}$ satisfies the KKT conditions of the problem \eqref{equ:9}.
\end{proof}

Let us discuss the assumptions in Theorem \ref{thm:5}. First, the feasible set $\mathcal{G}=\{\Phi\in\mathbb{R}^{d\times m}\mid\Phi\Phi^T=I_d,\Phi{\bf 1}_m=0\}$ is compact, so there always exists a sub-sequence $\{t_j\}_{j=1}^{\infty}$ such that $\{\Phi_{t_j}\}_{j=1}^{\infty}$ and $\{\Phi_{t_j-1}\}_{j=1}^{\infty}$ converge. If we assume $\widetilde\Phi_{t}$ is bounded, then we can choose $\{t_j\}_{j=1}^{\infty}$ such that $\widetilde\Phi_{t_j}$ also converges. Second, the assumption $\sigma_{\min}(T(\Phi^*)T(\Phi^*)^T)>0$ and $\sigma_{\min}(T(\Phi^{**})T(\Phi^{**})^T)>0$ are mild, especially when $m$ is much larger than $d$. 
Finally, the assumption \eqref{equ:50} 
characterizes the $\gamma$-strong convexity of $\ell(R^{**},\Phi)$ with respect to $\Phi$ surrounding the minimizer $\widetilde\Phi^{**}$. Theorem \ref{thm:2} provides conditions under which $g(\tau,\eta)$ is strongly convex over a bounded set, which implies the strong convexity of $\ell(R^{**},\Phi)$ with respect to $\Phi$ over a bounded set. 

\section{Regularized Quadratic Matrix Factorization}\label{sec:3}
{Motivated by the discussion at the end of Section \ref{sec:2.2}}, we propose a regularized quadratic matrix factorization (RQMF) method. {Specifically, we add a regularizer $\lambda\|Q\|_{\rF}$ to the original problem \eqref{equ:16} to prevent $\|Q\|_{\rF}$ from being too large. This leads to the following RQMF  problem:}
\begin{equation}
    \min_{R,\Phi\atop \Phi\Phi^T=I_d,\Phi{\bf 1}_m=0}\ell_{\lambda}(R,\Phi)=\|X-RT(\Phi)\|^2_{\rF}+\lambda\|RJ\|_{\rF}^2,\label{equ:39}
\end{equation}
where $R$ and $T(\Phi)$ are given by \eqref{equ:10},
\begin{equation}
    J=
    \begin{bmatrix}
    0 & I_{(d^2+d)/2}
    \end{bmatrix}^T\in\mathbb{R}^{\frac{2+3d+d^2}{2}\times\frac{d^2+d}{2}}, \textnormal{ and } RJ=Q.\label{equ:40}
\end{equation} 
Since the columns of $X$ are assigned equal weights in \eqref{equ:39}, we denote it as RQMF-E in the experiment section to distinguish it from its kernel version. In the following, we can still denote it as RQMF when not causing confusion.

As shall be seen later, RQMF with a proper $\lambda$ avoids over-fitting and thus enjoys a better generalization performance. In the rest of this section, we will first provide an alternating minimization algorithm to solve \eqref{equ:39} and then elaborate how to tune $\lambda$ properly.
Also, we show that no matter how $\lambda$ is chosen, RQMF always outperforms LMF in terms of memorization properties.

\begin{algorithm}[t]
\caption{Regularized Quadratic Matrix Factorization}
\label{alg:2}
\KwData{$X=[x_1,\ldots,x_m]\in\mathbb{R}^{D\times m}$}
\KwResult{$\Phi=[\tau_1,\ldots,\tau_m]\in\mathbb{R}^{d\times m}$, quadratic function $f(\tau)= R_t\xi(\tau)$.}
{
initialize $\Phi_0\in\mathbb{R}^{d\times m}$ as the top $d$ eigenvectors of $G=(X-\bar x{\bf 1}_m^T)^T(X-\bar x{\bf 1}_m^T)$\;
\While{$\|\Phi_{t}^T \Phi_{t}-\Phi_{t-1}^T\Phi_{t-1}\|>\epsilon$}
{update $R_{t} = \argmin_R \ell_\lambda(R, \Phi_{t-1})$ as in \eqref{equ:41}\;
{\For{$i=1$ \KwTo $m$}{solve the $i$-th projection problem $\widetilde\tau_{i,t}=\argmin_{\tau\in\mathbb{R}^{d}}\|x_i-R_t\xi(\tau)\|^2$\;}}
set $\widetilde\Phi_t=[\widetilde\tau_{1,t},\ldots,\widetilde\tau_{m,t}]$\;
update $\Phi_t=Z_t\widetilde\Phi_t(I_m-{\bf 1}_m{\bf 1}_m^T/m)$ with $Z_t$ given by \eqref{equ:19}\;
}
}
\end{algorithm}

To solve \eqref{equ:39}, we adopt the same alternating minimization strategy described in Algorithm~\ref{alg:1}. When $R$ is fixed, minimizing $\ell_{\lambda}(R,\Phi)$ with respect to $\Phi$ is equivalent to minimizing $\ell(R,\Phi)$ with respect to $\Phi$, since the regularizer $\lambda\|RJ\|_{\rF}^2$ is independent of $\Phi$. It thus reduces to the quadratic projection problem discussed in Section~\ref{sec:2.2}.
On the other hand, when $\Phi$ is fixed, minimizing $\ell_\lambda(R,\Phi)$ with respect to $R$ is a ridge regression problem, and the solution can be given in closed form as
\begin{equation}
\begin{aligned}
    \widetilde R=\ &\argmin_{R}\ell_{\lambda}(R,\Phi)\\
    =\ &XT(\Phi)^T(T(\Phi)T(\Phi)^T+\lambda
    JJ^T)^{-1}.\label{equ:41}
\end{aligned}
\end{equation}
Here we use the observation that $T(\Phi)T(\Phi)^T+\lambda JJ^T$ is invertible as shown in Lemma \ref{lma:1} below. Therefore, to solve \eqref{equ:39}, it suffices to replace \eqref{equ:13} in Algorithm~\ref{alg:1} by \eqref{equ:41}, which gives us Algorithm~\ref{alg:2}, the RQMF algorithm.

\begin{lemma}\label{lma:1}
    Suppose $\Phi\Phi^T=I_d$ and $\Phi{\bf 1}_m=0$. If $\lambda>0$, then $T(\Phi)T(\Phi)^T+\lambda JJ^T$ is positive definite. Furthermore, $J^T(T(\Phi)T(\Phi)^T+\lambda JJ^T)^{-1}J$ is also positive definite.
\end{lemma}

The proof of Lemma~\ref{lma:1} is left in the Appendix.
The following proposition shows that no matter how $\lambda$ is chosen, RQMF memorizes the data better than LMF. Recall that LMF corresponds to RQMF with $\lambda\to\infty$, or equivalently  $Q=RJ=0$.

\begin{proposition}\label{prop:7}
For any $\lambda> 0$, the RQMF that solves \eqref{equ:39} memorizes the data better than LMF in the following sense:
\begin{equation}
\|X-R^*T(\Phi^*)\|_{\rF}^2\leq  \|X-R'T(\Phi')\|_{\rF}^2,
\end{equation}
where $(R',\Phi')=\argmin_{R\in\Omega,\Phi}  \ell_\lambda(R,\Phi)$ with $\Omega = \{R\mid R = [c, A, Q], Q=0\}$ is the solution of LMF and
$(R^*, \Phi^* )= \argmin_{R,\Phi} \ell_\lambda(R,\Phi)$ is the solution of RQMF. 
\end{proposition}
The proof of Proposition~\ref{prop:7} is left in the Appendix.

\subsection{Tuning Parameter Selection}\label{sec:3.1}
\begin{figure*}[t] 
\centering
\includegraphics[width=.9\linewidth]{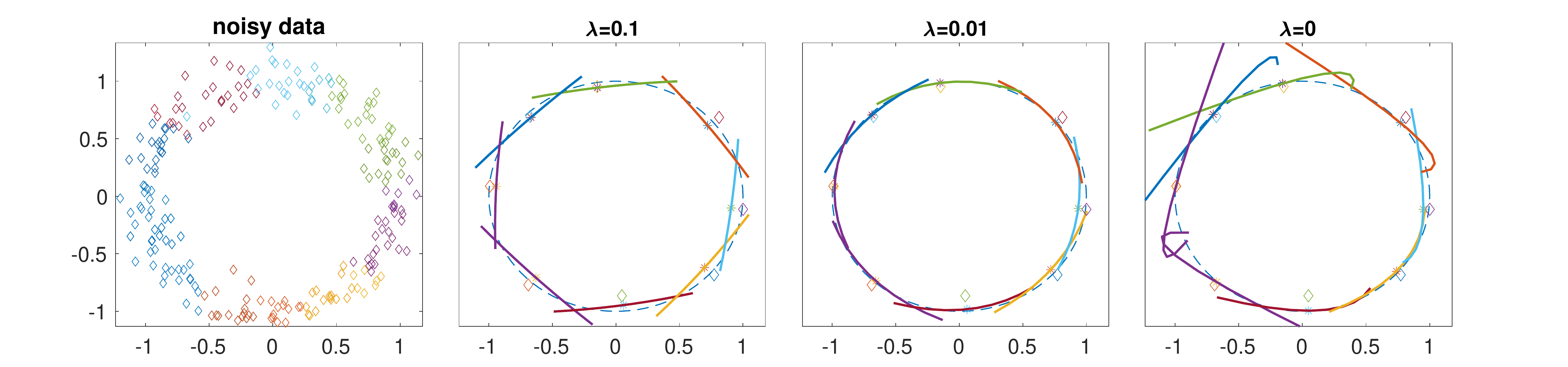}
\caption{Illustration of the effect of $\lambda$ for fitting a circle. This figure displays the generated data and the locally fitted curves with $\lambda=0.1, 0.01, 0$  (from left to right). In the last three figures, the rhombuses stand for the target samples, and the asterisks represent the place where the target samples are projected to fitted curves.}
\label{fig:2}
\end{figure*}

{This section discusses 
how to tune $\lambda$.} 
Before presenting our new tuning method, let us first illustrate the effect of $\lambda$ in Figure \ref{fig:2}. 
We generate 240 data points uniformly on the unit circle and then manually add normal noises obeying ${\cal N}(0,0.1^2 I)$. For each target sample, we fit a curve around the target data using the nearest 40 data points and then project the target data onto the fitted curve. To fit the curve, we use the RQMF algorithm with $\lambda=0.1$, 0.01, and 0. 
{Figure \ref{fig:2} shows that} the RQMF algorithm with $\lambda=0.01$ achieves the best performance. When $\lambda=0.1$ is too large, the RQMF algorithm behaves like linear matrix factorization and tends to use straight lines as the fitted curves. When $\lambda=0$ is too small, the RQMF algorithm tends to overfit data with excessively curved lines. Therefore, it is important to pick a proper $\lambda$.


In what follows, we will describe a new adaptive tuning method.
Recall that when $\Phi$ is fixed, $\widetilde R$ in \eqref{equ:41} is a function of $\lambda$. Define $s(\lambda)=\|\widetilde R(\lambda)J\|_{\rF}^2$ and denote by $s'(\lambda)$ and $s''(\lambda)$ the corresponding first and second derivatives.  
\begin{figure}[t] 
\centering
\includegraphics[width=\linewidth]{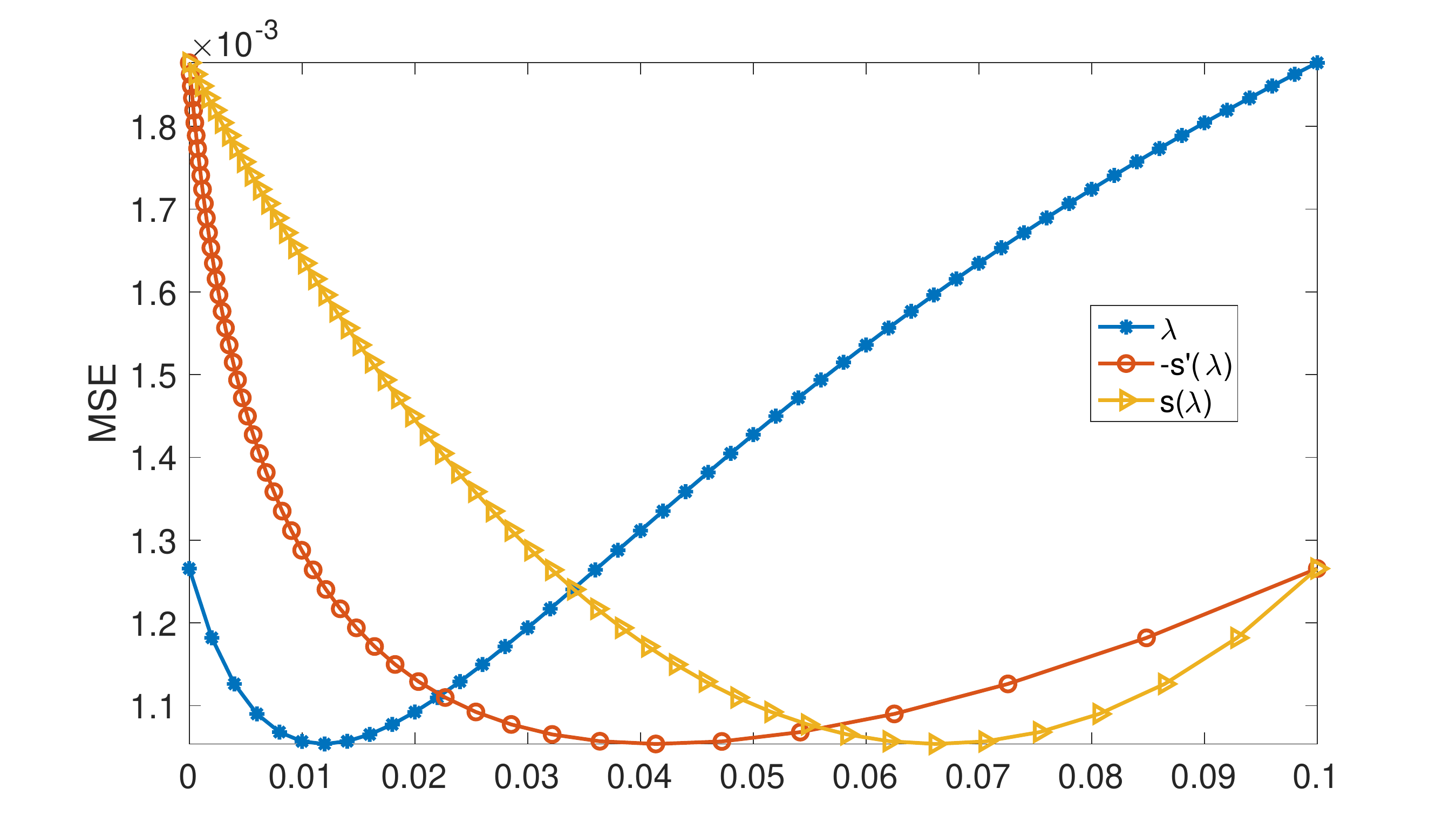}
\caption{The fitting error of RQMF against different $\lambda$, $\delta(\lambda)=-s'(\lambda)$, and $s(\lambda)$. Here we take $\lambda\in[0,0.1]$ and compute $\delta(\lambda)$ and $s(\lambda)$.  
To compare these three curves in the same horizontal axis $[0,0.1]$, we shift and re-scale $\delta(\lambda)$ and $s(\lambda)$ via the function: $f(x)= 0.1\cdot\frac{x-\min(x)}{\max(x)-\min(x)}$. 
}
\label{fig:3}
\end{figure}

Proposition~\ref{prop:2} shows that $s'(\lambda)<0$ and $s''(\lambda)>0$ when $s(\lambda)>0$ and $\lambda>0$. This implies that $s(\lambda)$ is a decreasing function of $\lambda$ while $s'(\lambda)$ is a strictly increasing function of $\lambda$. In particular, the root $\lambda=s'^{-1}(-\delta)$ is unique and can be easily found via the bisection method.
We propose to pick $\lambda=s'^{-1}(-\delta)$ for a prescribed $\delta>0$.

The proposed tuning method is 
reasonable in the following sense. Recall that $s(\lambda)$ is the quantity that the regularizer $\|RJ\|_{\rF}^2$ aims to control and $s'(\lambda)$ measures the sensitivity of the target quantity $s(\lambda)$ with respect to $\lambda$. Thus, our tuning method chooses $\lambda$ corresponding to a prespecified sensitivity level $\delta$.

To illustrate the advantage of {tuning} $\lambda$ via $s'^{-1}(\cdot)$, {we 
implement RQMF 
with 50 different $\lambda$  evenly spaced in $[0,0.1]$.  We use 
samples drawn from the sine curve, that is, $(t_i, \sin(t_i))+\epsilon_i$ with $\{t_i\}_{i=1}^{21}$ evenly distributed in
$[\frac{\pi}{3}, \frac{2\pi}{3}]$ 
and $\epsilon_i\overset{\rm i.i.d.}{\sim} {\cal N}({\bf 0}, 0.03^2 I)$.
For each $\lambda$, we implement RQMF and compute the error, that is, the average distance between the fitted data points and the underlying truth. Also, we calculate the values of $s(\lambda)$ and $\delta(\lambda)=-s'(\lambda)$. Figure \ref{fig:3} displays the error against different $\lambda$, $s(\lambda)$, and $\delta(\lambda)$.
It shows that the error versus $\delta(\lambda)$ curve is the flattest near the optimal error. To achieve a prespecified error, say $1.2\times 10^{-3}$, the feasible choice of $\delta(\lambda)$ has a much wider range than that of $\lambda$ or $s(\lambda)$. 
Thus, it is easier to achieve good performances of RQMF by choosing $\delta(\lambda)$ rather than $\lambda$ or $s(\lambda)$. 
In practice, we choose $\delta>0$ as a constant smaller than $-s'(0)$, where $\Phi$ determining the function $s(\cdot)$ is given by LMF.
}

\begin{proposition}\label{prop:2}
    Suppose $\Phi$ is fixed with $\Phi\Phi^T=I_d$ and $\Phi{\bf 1}_m=0$. Define $\widetilde R(\lambda)$ by \eqref{equ:41} and  $s(\lambda)=\|\widetilde R(\lambda)J\|_{\rF}^2$. Then $s'(\lambda)\leq0$ and $s''(\lambda)\geq0$. Furthermore, the strict inequalities $s'(\lambda)<0$ and $s''(\lambda)>0$ hold if $s(\lambda)>0$ and $\lambda>0$.
\end{proposition}
The proof of Proposition~\ref{prop:2} is collected in the Appendix.

\section{Applications to Manifold Learning} \label{application}


This section applies the RQMF algorithm to manifold learning problems.
Assume data $\{x_i\}_{i=1}^m\subseteq\mathbb{R}^D$ are generated near an unknown smooth manifold $\cal M$ of intrinsic dimension $d$. Here we no longer assume all data belong to the same local chart. Instead, we assume data belong to a union of several local charts.
To recover the underlying manifold, we can apply the RQMF algorithm for each local chart.


The performance of this divide-and-conquer strategy depends on choices of specific local charts. Besides, the quality of the fitted points on a single chart cannot be guaranteed uniformly: recovering the central region tends to be of higher quality than recovering the marginal region. Also, for a data point belonging to multiple local charts, the fitted points in different charts are different and it is hard to determine which one is the best. To address these challenges, we propose an improved divide-and-conquer strategy. This strategy constructs a local chart for each  data point $y$ by finding its nearest $K$ samples or by $\mathcal{N}(y,a)=\{x_i\mid\|x_i-y\|\leq a\}$ for some $a>0$. Then for each target sample, we denoise this particular data point by applying the RQMF algorithm to the corresponding chart. Compared with the original divide-and-conquer strategy, our strategy treats each data point as an individual problem and improves the accuracy.

{ In a more general form, we could use a kernel function $K_h(\cdot,\cdot)$ to assign a closer point with higher importance, where $h$ is the bandwidth. For a target sample $y$, we modify the loss function in \eqref{equ:39} as
\begin{equation}
\begin{aligned}
    \min_{R,\Phi\atop \Phi\Phi^T=I,\Phi{\bf 1}_m^T={\bf 0}}\ell_{\lambda,y,h}(R,\Phi)=&\|(X-RT(\Phi))W_h^{1/2}(y)\|_{\rF}^2\\
    &+\lambda\|RJ\|_{\rF}^2,\label{equ:62}
\end{aligned}
\end{equation}
where $X=(x_1,\ldots,x_m)$ is the {\it global} data matrix and $W_h^{1/2}(y)\in\mathbb{R}^{m\times m}$ is a diagonal weight matrix with the $i$-th diagonal element equal to $K_h^{1/2}(x_i,y)$. If we choose $K_h(x,y)=1_{\|x-y\|\leq a}$, then \eqref{equ:62} reduces to the improved divide-and-conquer strategy mentioned above. It is also possible to use other kernel functions, such as the Gaussian kernel. 
To distinguish the kernel RQMF model from the previous equal-weight RQMF, we use RQMF-E to represent equal-weigth RQMF and  RQMF-K to represent RQMF with weights determined by a kernel.
}

We also discuss how to tune $\lambda$ for each sub-problem \eqref{equ:62}.
Picking the same $\lambda$ for all sub-problems is not desirable since the best $\lambda$ depends on the weights $W_h^{1/2}(y)$ and the curvature of the underlying truth, which vary as $y$ changes.
Instead, we suggest using the tuning method proposed in Section \ref{sec:3.1}, which picks $\lambda=s'^{-1}(-\delta)$ for the same prescribed sensitivity level $\delta>0$ for all sub-problems. 
The same $\delta$ would result in different $\lambda$'s for different charts, and this strategy often leads to better fitting accuracies in our experience. 

\begin{table*}[h]
\centering
\caption{Comparisons of different methods on the synthetic spherical dataset in terms of  MSE and SD (in bracket) with varying $K$. \label{tab:2}}
\centering
\resizebox{18cm}{!}{
\begin{tabular}{c | c| c| c| c| c| c |c |c} 
\hline \hline
$K$           & {7} & {10}   & {13} & {16} & {19} &22 & 25& 28 \\ 
\hline
RQMF-E   
& \underline{0.0243} (\underline{0.0325}) &\bf{0.0165} ({\bf 0.0227}) & {\bf 0.0122} ({\bf 0.0172}) & {\bf 0.0115} ({\bf 0.0164}) & {\bf 0.0148} ({\bf 0.0256}) & {\bf 0.0130} ({\bf 0.0222}) & {\bf 0.0149} ({\bf 0.0344}) & {\bf 0.0156} ({\bf 0.0356}) \\
RQMF-K
& {\bf 0.0188} ({\bf 0.0281}) & \underline{0.0170} (\underline{0.0270}) &\underline{0.0155} (\underline{0.0231}) & \underline{0.0148} (\underline{0.0233}) & \underline{0.0153} (\underline{0.0242}) & \underline{0.0159} (\underline{0.0252}) & 0.0170 (\underline{0.0266}) &\underline{0.0185} (\underline{0.0280}) \\
Local PCA 
& 0.0437 (0.0521) & 0.0437 (0.0522) & 0.0435 (0.0520) & 0.0432 (0.0519) & 0.0434 (0.0521) & 0.0434 (0.0518) & 0.0434 (0.0517) & 0.0434 (0.0517) \\
KDE
 & 0.0333 (0.0480) & 0.0298 (0.0469) & 0.0302 (0.0483) & 0.0307 (0.0482) & 0.0323 (0.0493) & 0.0342 (0.0499) & 0.0369 (0.0492) & 0.0389 (0.0501) \\
LOG-KDE 
& 0.0278 (0.0417) & 0.0192 (0.0350) & 0.0159 (0.0344) & 0.0155 (0.0348) & 0.0168 (0.0349) & 0.0192 (0.0370) & 0.0230 (0.0395) & 0.0275 (0.0415) \\
Mfit
& 0.0392 (0.0463) & 0.0333 (0.0424) & 0.0262 (0.0347) & 0.0215 (0.0314) & 0.0183 (0.0302) & 0.0160 (0.0299) & \underline{0.0154} (0.0303) & \underline{0.0185} (0.0376) \\
Moving LS
& 0.0420 (0.0844) & 0.0673 (0.1155) & 0.1017 (0.1474) & 0.1506 (0.1666) & 0.1898 (0.1819) & 0.2264 (0.2033) & 0.2544 (0.2060) & 0.2642 (0.2072) \\ \hline\hline
\multicolumn{9}{r}{\tiny Numbers in bold and underlined are the best and second-best results for each column's setting, respectively.}
\end{tabular}}
\end{table*}

\section{Numerical Experiments}\label{experiment}
This section presents numerical experiments on a synthetic manifold learning dataset and two real datasets, including the MNIST handwritten dataset and a cryogenic electron microscopy dataset,  to examine the finite-sample performance of the proposed method. Our goal is to reconstruct the underlying manifold from noisy data and compare our method with five commonly used methods. We first briefly describe these five competitors. 
\begin{itemize}
    \item {\bf Local PCA} For any $x$, it first finds the $K$ nearest data points $\{x_{i_1},...,x_{i_K}\}$ and then computes the covariance matrix $M = \frac{1}{K}\sum_{k = 1}^K (x_{i_k} -c_x)(x_{i_k}-c_x)^T\in\mathbb{R}^{D\times D}$, where $c_x = \sum_{k=1}^K  x_{i_k}/K$ is the center of these samples. Denote by $P\in\mathbb{R}^{D\times D}$ the projection matrix corresponding to the space spanned by the $d$ principle eigenvectors of $M$.
    The denoised point of $x$, a point on the estimated manifold ``projected" from $x$, is given by $x_{\textnormal{new}}=c_x+P(x-c_x)$.

\item {\bf KDE \& LOG-KDE Ridge Estimation} Both methods are special cases of the nonparametric ridge estimation method \citep{genovese2014nonparametric}. Let $\hat p(x)=\sum_{i}K_h(x,x_i)$ be the kernel density estimation (KDE) with the kernel function $K_h(\cdot,\cdot)$ and the bandwidth $h$. KDE ridge estimator estimates the ridge: 
\begin{equation}\label{equ:63}
\begin{aligned}
{\rm ridge}\coloneqq \{x\mid &\Pi^{\perp}(\nabla^2\hat{p}(x)) \nabla \hat{p}(x) = {\bf 0},\\
& \lambda_{d+1}(\nabla^2 \hat{p}(x))<0\},
\end{aligned}
\end{equation}
where $\Pi^{\perp}(\nabla^2 \hat{p}(x)) = I-UU^T$ with $U\in\mathbb{R}^{D\times d}$ given by the top $p$ principal eigenvectors of $\nabla^2 \hat{p}(x)$. Although the ridge in \eqref{equ:63} does not admit close-form solutions, we may use the subspace constrained mean shift (SCMS) algorithm to find the denoised point of any point $x$ and thus the ridge \citep{ozertem2011locally}. Similarly, the LOG-KDE ridge estimation is merely KDE ridge estimation with $\hat p(x)$ replaced by $\log \hat p(x)$.

\item {\bf Mfit} Mfit, proposed by \citep{fefferman2018fitting}, estimates the following manifold:
\$
\left\{x\mid\Pi_x\big(\sum_{i}\alpha(x,x_i)\Pi_i(x-x_i)\big)={\bf 0}\right\},
\$
where $\alpha(\cdot,\cdot)$ is a weight function, $\Pi_i$ is the projection matrix onto the approximate normal space at $x_i$, and $\Pi_x$ is the projection matrix corresponding to the top $D-d$ principal eigenvectors of the matrix $\sum_{i}\alpha(x,x_i)\Pi_i$. Again, we may use the SCMS algorithm to solve this problem.

\item {\bf Moving LS} The moving least square (LS)  consists of two steps \citep{sober2020manifold}. For any $x\in\mathbb{R}^D$, we first find the $d$-dimensional hyperplane $\cal H$ in $\mathbb{R}^{D}$
minimizing the following quantity
\$
{\cal L}_1({\cal H})=\min_{q\in{\cal H},x-q\perp {\cal H}}\sum_{i} \alpha(q, x_i) \rho^2(x_i, {\cal H}),
\$ 
where $\alpha(\cdot,\cdot)$ is a weight function and $\rho(x_i,{\cal H})$ is the  distance between $x_i$ and 
${\cal H}$.
We can construct a coordinate system on $\cal H$ with origin $q$, where $q$ is the projection point of $x$ on $\cal H$.
Using this coordinate system, we can obtain the $d$-dimensional configuration  $x_i'$ of $x_i$ by projecting $x_i$ to $\cal H$. Second, we fit a polynomial function $p:\mathbb{R}^d\to\mathbb{R}^D$ of a given degree by {minimizing the following weighted squares}
\$
{\cal L}_2(p)=\sum_i\alpha(q,x_i)\|p(x_i')-x_i\|_2^2.
\$
The denoised point of $x$ is then given by $p(0)$. In our experiments, we fix the degree of $p$ as two.

\end{itemize}

 \begin{figure}[t]
   \centering
   \includegraphics[width=0.5\linewidth]{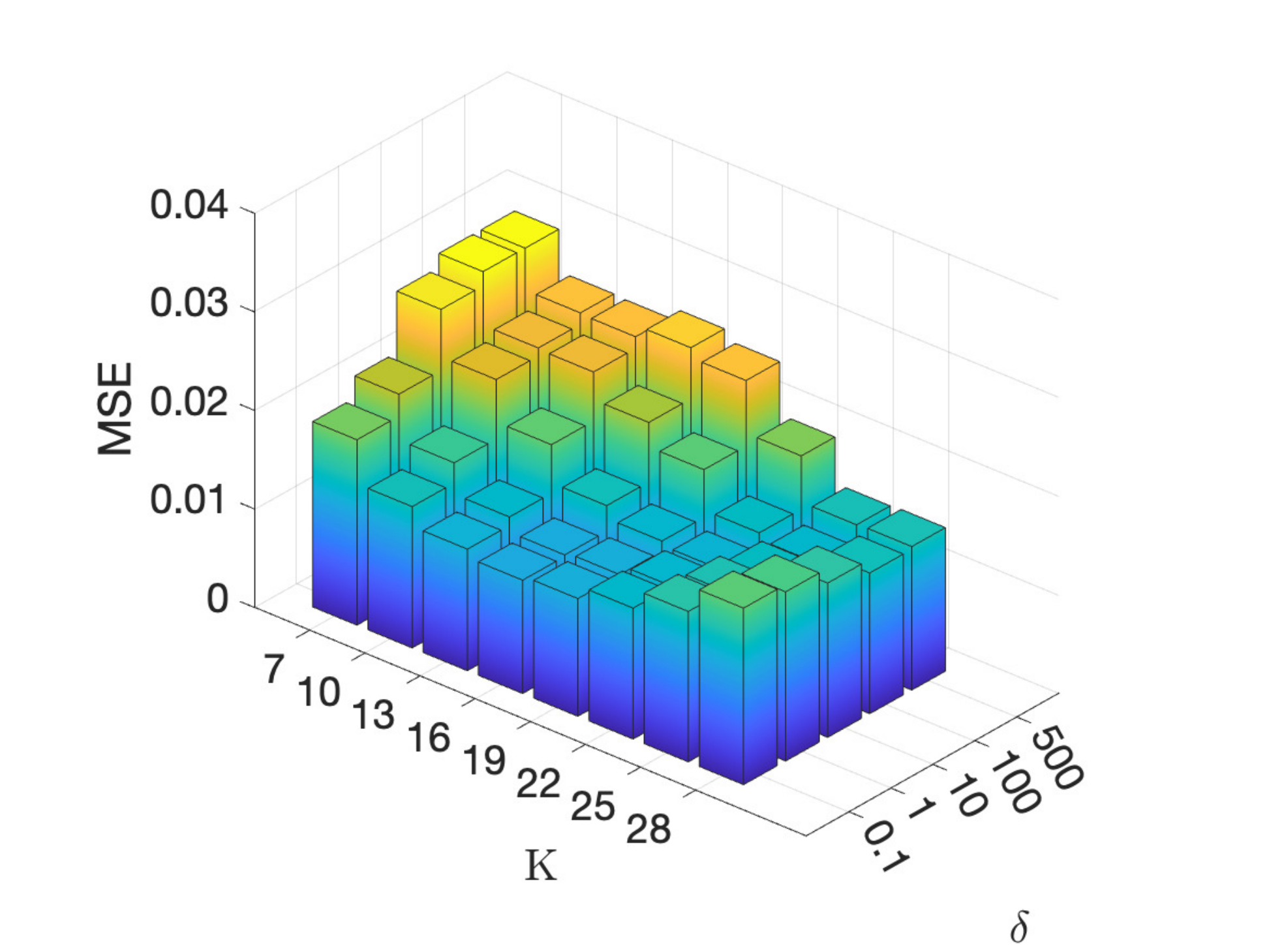} 
   \hspace{-1cm}
   \includegraphics[width=0.5\linewidth]{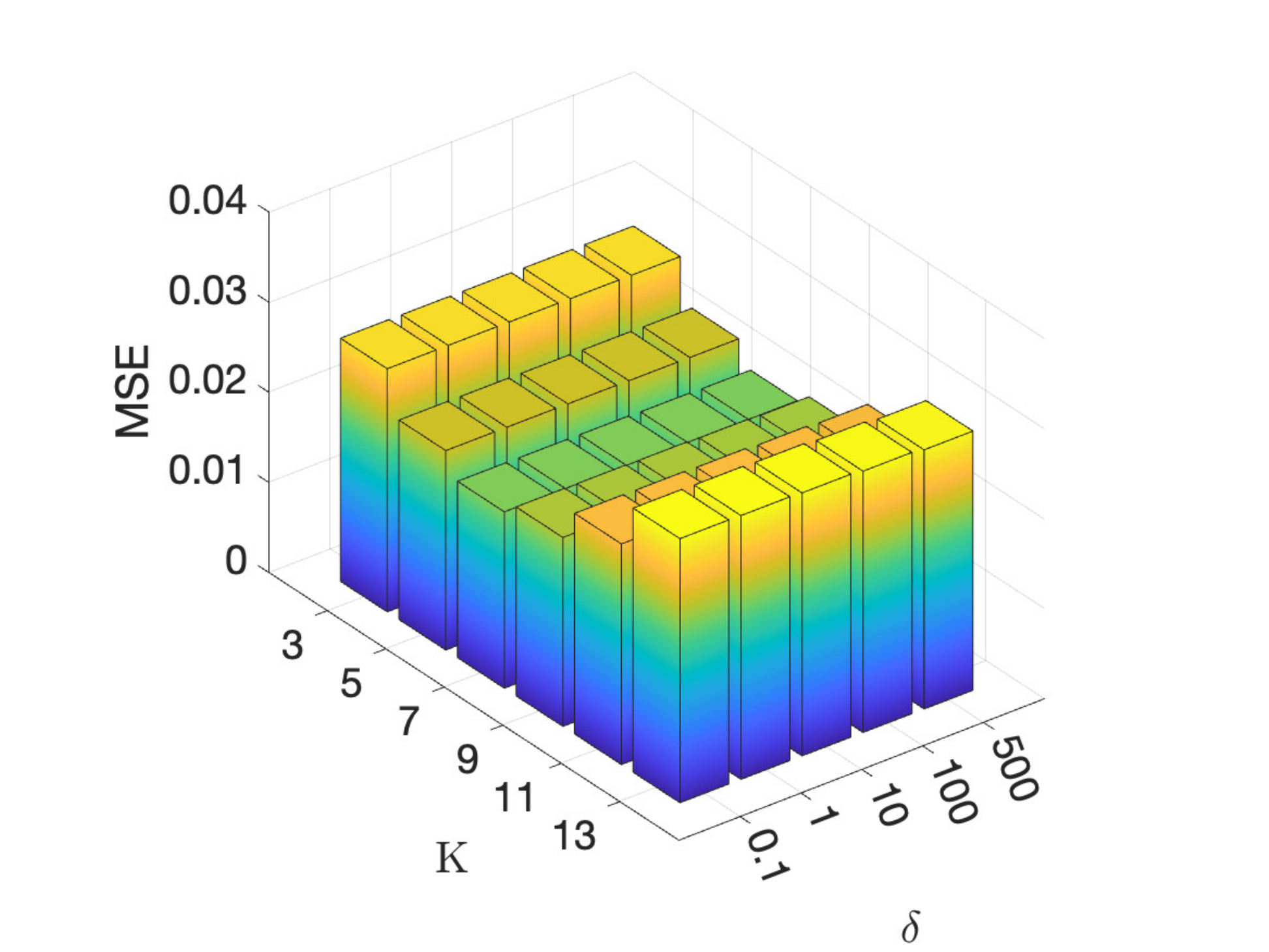} 
   \caption{An illustration of the impact of $\delta$ and $K$ for spherical fitting for RQMF-E (left) and RQMF-K(right).   \label{fig:K-delta-MSE}}
\end{figure}

\begin{figure}[h]
\centering
   \includegraphics[width=\linewidth]{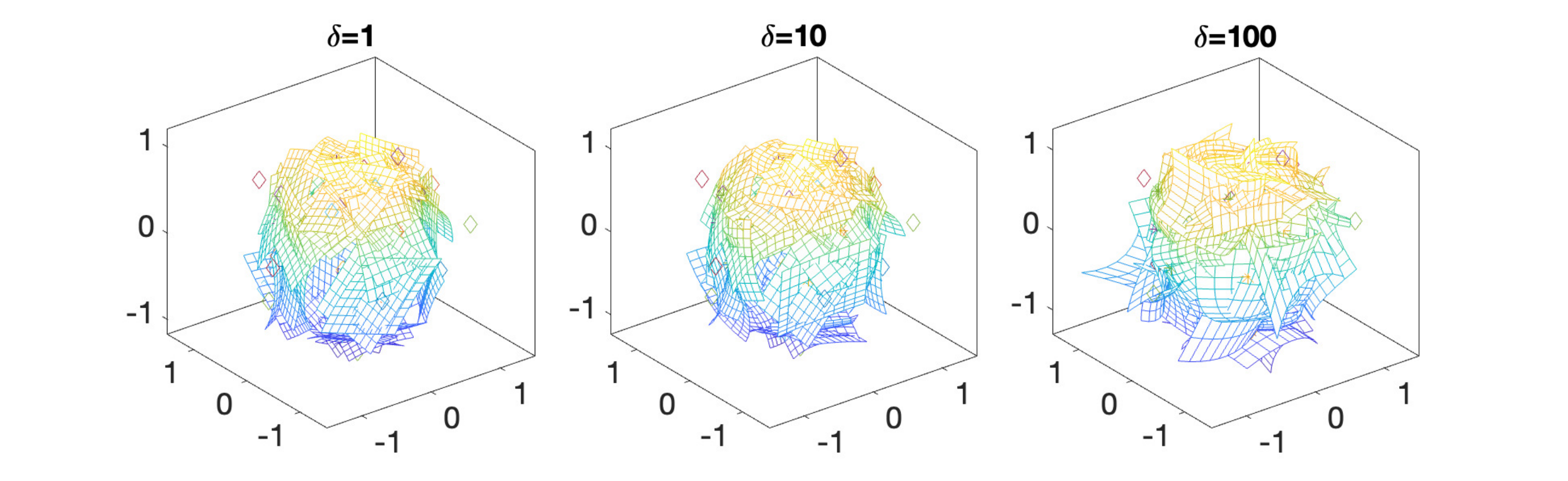}
   \caption{An illustration of local fitted surface under the impact of $\delta$ for RQMF-E when $K = 18$.   \label{fig:RQMF-E-delta}}
\end{figure}

\subsection{A Synthetic Example}\label{sec:6.1}

{
{In this subsection, we compare RQMF-E and RQMF-K with the above five competitors  in a synthetic spherical fitting experiment. We simulate  the noisy data $\{x_i\}_{i=1}^{240}$ by generating 240 points uniformly from the unit sphere $\cal S$ in $\mathbb{R}^{3}$ first and then adding independent noises following ${\cal N}(0, \sigma^2 I)$ with $\sigma=0.2$. All algorithms take in the noisy data $\{x_i\}$ and then  output the denoised data  $\{\widehat{x}_i\}$. To measure the performance of different algorithms, we use the mean squared error (MSE) and standard derivation (SD):
 \begin{gather}
     {\rm MSE} = \sum_{i=1}^m \|\widehat x_i-P_{\cal S}(\widehat x_i)\|^2_{2}/m,\notag\\
 {\rm SD} = \sqrt{\frac{1}{m}\sum_{i=1}^m (\|\widehat x_i-P_{\cal S}(\widehat x_i)\|_2^2-{\rm MSE})^2},\notag
 \end{gather}
 where $P_{\cal S}(\cdot)$ is the projector onto the sphere. 

\begin{figure*}[t] 
   \centering
   \includegraphics[width=0.8\linewidth]{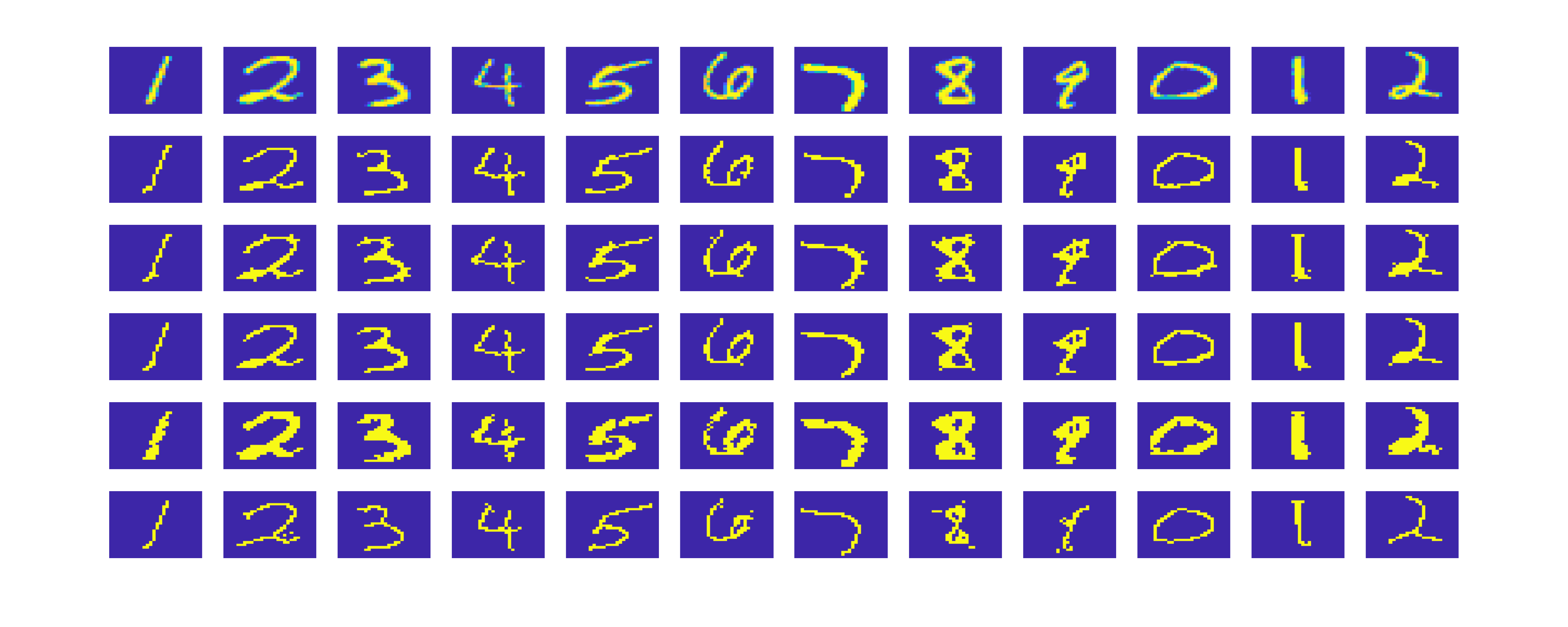} 
   \caption{The first row displays 12 examples from the original MNIST dataset. The second to the sixth rows collect the results for the RQMF-K, KDE, LOG-KDE, Mfit, and Moving LS algorithms, respectively.\label{fig:MNIST RESULT}}
\end{figure*}
\begin{table*}[h!]
\centering
\caption{Comparison of the smoothness measured by the average of the nonzero of the absolute value of $I*w$ corresponding to 12 example images \label{Pixel smoothness}}
\begin{tabular}{c|c|c|c|c|c|c|c|c|c|c|c|c}\hline \hline
Image ID& 1& 2& 3 & 4 & 5 & 6 & 7 & 8 & 9 & 10 & 11 & 12\\ \hline
 Original& \underline{1.0637} & 1.2264 & 1.0306 & \underline{0.8760} & 1.0332 & 0.9499 & 1.0246 & 1.0261 & 1.0068 & 1.0267 & 0.7178 & 1.0326 \\
RQMF-K& \bf 0.4848 & \bf  0.5795 &\bf  0.5168 & \bf 0.5785 & \bf 0.5440 & \bf 0.4851 & \bf 0.5733 & \bf 0.5319 & \bf 0.6011 & \bf 0.5388 & \bf 0.5327 &\bf 0.5470 \\
KDE& 1.0875 & 0.8876 &\underline{ 0.7574} & 0.9382 & \underline{0.8387} & \underline{0.9246} & \underline{0.7489} & \underline{0.6942} & \underline{0.6687} & 0.9070 & 0.6320 &\underline{0.7359} \\
LOG-KDE& 1.1395 & \underline{0.8861} & 0.7732 & 0.9460 & 0.8606 & 0.9329 & 0.7650 & 0.7523 & 0.6979 & 0.8812 & 0.6361 & 0.7764 \\
Mfit& 1.1769 & 0.9158 & 0.8567 & 0.9026 & 0.9520 & 0.9705 & 0.7956 & 0.8274 & 0.9107 & \underline{0.8562} & \underline{0.6189} & 0.7976 \\
Moving LS& 1.1926 & 1.3185 & 1.4126 & 1.0072 & 1.2117 & 1.1529 & 1.4050 & 1.3821 & 1.4752 & 1.2451 & 0.9490 & 1.4394 \\ \hline \hline
\multicolumn{13}{r}{\tiny Numbers in bold and underlined are the best and second-best results for each column's setting, respectively.}
\end{tabular}
\end{table*}

We briefly discuss how  RQMF-E and RQMF-K are implemented.
RQMF-E denoises each data point $y$ using the $K$ nearest neighbors of $y$, where $K$ is a tuning parameter.
To avoid overfitting, we require $K$ to be larger than the rank $(d^2+3d+2)/2$ of $RT(\Phi)$. 
For RQMF-K, we use the Gaussian kernel and set the bandwidth as $h=d_K/3+3$, where $d_K$ is the distance from $y$ to its $K$-th nearest neighbor. 
For both RQMF algorithms, we set the regularization parameter $\lambda$ as $\lambda= s'^{-1}(-\delta)$, where $\delta$ is a tuning parameter.
To select $K$ and $\delta$ for both RQMFs, we compute the MSEs of both RQMF-E and RQMF-K with different  $(K, \delta)$ as shown in Figure~\ref{fig:K-delta-MSE}. {When $K$ is small, the local data approximates a plane and thus a smaller $\delta$ (larger $\lambda$) yields a better performance for RQMF-E.
When $K$ is large, the local data exhibits nonlinear structures, thus it is better to use a larger $\delta$ (smaller $\lambda$) for RQMF-E.} The effect of $\delta$ is visualized in Figure~\ref{fig:RQMF-E-delta}. It shows that smaller $\delta$ tends to fit flatter planes in comparison with larger $\delta$, which coincides with the phenomenon in Figure~\ref{fig:K-delta-MSE}. 
To characterize the good performance region for $\delta$ and $K$ in Figure~\ref{fig:K-delta-MSE}, we choose $\delta=\max\{1,8K-125\}$
to determine $\delta$ based on $K$
for RQMF-E in this experiment. On the other hand, the performance of RQMF-K is relatively robust to the choice of $\delta$, thus we fix $\delta=100$ for different choices of $K$.
}

}

{

Now we compare RQMF-E and RQMF-K with their competitors. The results are collected in Table~\ref{tab:2}.
The results indicate that RQMF-E and RQMF-K outperform other methods for a wide range of $K$. 
When $K=7$, RQMF-K achieves the best performance and when $10<K<28$, RQMF-E achieves the best performance among all methods. 
If we focus on the best performance of different algorithms, the RQMF-E is still favored with the minimal ${\rm MSE}=0.0115$ when $K=16$. 
The superior performances of RQMF demonstrate the benefits of using the curvature information in the denosing procedure.
It is also worth noting that RQMF outperforms Moving LS, which also fits a quadratic polynomial in its second step. 
This is possibly due to fact that RQMF iteratively updates the local representations $\Phi$ of data points, while Moving LS only uses the local coordinates learned in its first step. The estimation error of the local coordinates learned in  Moving LS could lead to a degradation of the final fitting accuracy. 
}

\subsection{An Application to the MNIST Handwritten Digit Dataset}

This subsection compares RQMF-K and 
its competitors on the MNIST handwritten digit dataset \citep{deng2012mnist}. Each image in the dataset consists of $28\times 28$ pixels. 
We use $g(a,b)$ to denote the grey value of an image at pixel $(a,b)$ and each image is determined by such a function $g$.
Only pixels with nonzero grey values are considered, so the dataset for each image is given by
$\{x_i=(a_i,b_i)\in\mathbb{R}^2\mid g(a_i,b_i)>0\}$. In this way, each image can be viewed as a perturbed one-dimensional manifold in $\mathbb{R}^2$ and
our goal is 
to recover 
the underlying 
manifold, 
which is also 
 referred to as the principal curve 
\citep{ozertem2011locally}. 

 The pixel closer to the principal curve tends to have a larger grey value. Thus, 
 it is natural to use the grey values as weights in \eqref{equ:62}. Specifically, for each $y$, we set the diagonal weight matrix
 $W_h(y)$ in \eqref{equ:62} 
 by $(W_h(y))_{ii}=K_h(y,x_i)g(x_i)/s$ for some constant $s$, where 
 $K_h(\cdot, \cdot)$ is a Gaussian kernel and the bandwidth $h$ is
 given by the distance 
 of $y$ to $y$'s $K$-th nearest pixel.
 We use $\delta=100$ to tune the parameter $\lambda=s'^{-1}(\delta)$.
 
Since the smoooth curve contains the most significant signal and the isolated points can be thought as noise, 
we measure the smoothness 
of the image using the convolution of the original image with the Laplace operator 
\$w = 
\left [
\begin{array}{ccc}
0 & 1 & 0\\
1& -4 & 1\\
0 & 1 & 0
\end{array}
\right].
\$
The obtained matrix $I*w$ is a discrete version of the Laplace operator defined for function $f$, i.e., $\Delta^2 f(x,y) = \frac{\partial^2 f}{\partial x^2}+ \frac{\partial^2 f}{\partial y^2}$. The average value of   $I*w$ represents the degree of smoothness of the image $I$. We report the average of the nonzero of the absolute value in $I*w$ in Table \ref{Pixel smoothness}.


\begin{figure*}
\centering
   \includegraphics[width=\linewidth]{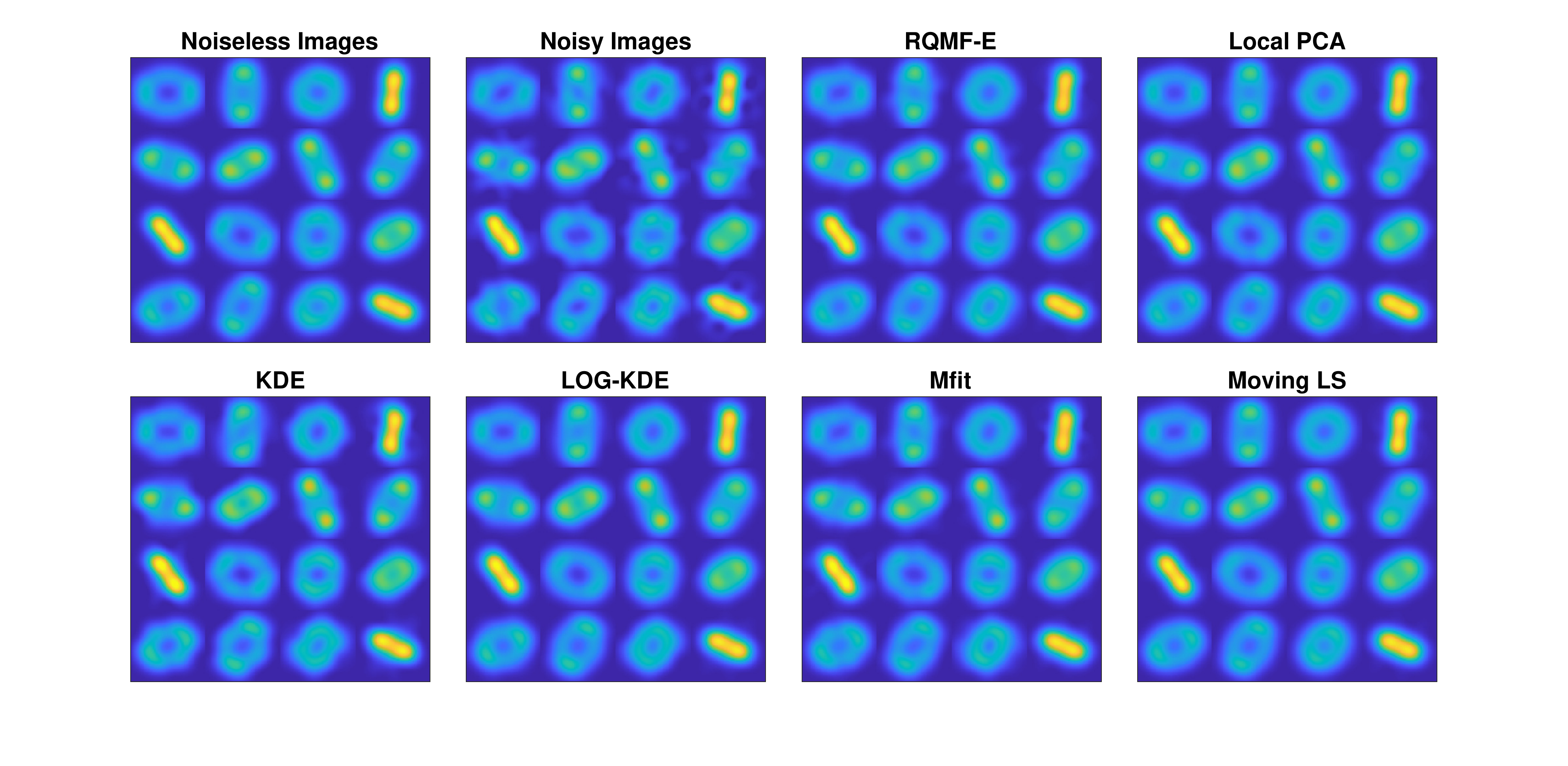} 
   \caption{Visualization the fitting and denoised results of 16 randomly chosen images by RQMF-E and five related manifold learning methods with $K=60$. \label{fig:6}}
\end{figure*}

For each image, we apply all six manifold learning algorithms to recover the principal curve, and the results are displayed in Figure \ref{fig:MNIST RESULT} and Table \ref{Pixel smoothness}.
{It turns out that the denoised images by RQMF-K are smoother than the original images and images output by KDE, LOG-KDE, Mfit. While Moving LS also produces smoother images, it tends to twist the original images too much; see images of 6,8,9. In contrast,
RQMF preserves the major contour of the original images much better.}

\subsection{An Application to Cryo-EM}

This subsection compares RQMF-E and RQMF-K with its competitors
on the Cryo-EM dataset 
\citep{bai2015cryo}.
This dataset consists of $n=2000$ images with shape $64\times 64$. Each image is modeled as a vector in $\mathbb{R}^{4096}$ with elements given by the grey values on all 4096 pixels. The whole dataset is then represented as $\{\iota_i\}_{i=1}^{m}\subseteq
\mathbb{R}^{4096}$,
where $\iota_i$ denotes the $i$-th image. 
The dataset inherently resides on a lower-dimensional manifold in $\mathbb{R}^{4096}$ due to the image generation and processing of Cryo-EM, such as rotation, projection, and blurring by convolution \citep{fefferman2018fitting}. In our experiment, we take the original data as the underlying manifolds, add noises to the original data, 
apply all six manifold learning algorithms to the noisy data, and finally compare their recovery accuracies. 

\begin{figure}[t] 
   \centering
   \includegraphics[width=0.8\linewidth]{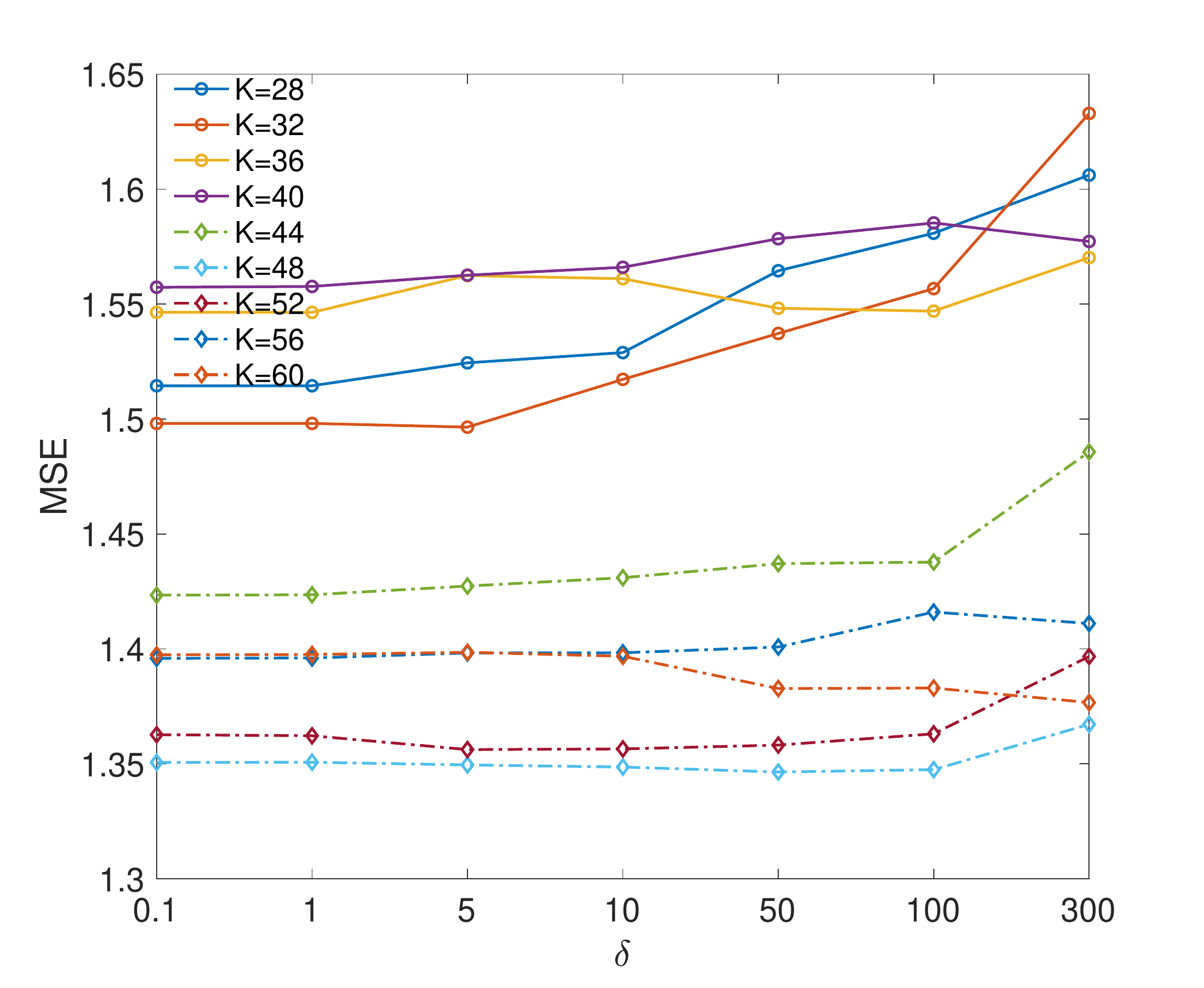} 
   \caption{The performance of RQMF-E  with different $K$ and $\delta$ for Cryo-EM. \label{cryo_K_delta}}
\end{figure}

\begin{table*}[t]
\centering
\caption{
Comparisons of six manifold learning algorithms on the Cryo-EM dataset. We report 
MSE and SD (in bracket) with varying  
$K$.\label{tab:3}}
\begin{tabular}{c|c|c|c|c|c|c}\hline\hline
$K$ & 40 &44 &48&52 &56 & 60\\\hline
RQMF-E& \underline{1.5757}(0.7065) & {\bf 1.4203}(0.6669) & {\bf 1.3482}(0.6013) & {\bf 1.3791}(0.6363) & {\bf 1.4134}(0.6465) & {\bf 1.3764}(0.6623) \\
RQMF-K& 2.1407(1.1057) & 2.1008(1.0804) & 2.0718(1.0612) & 2.0310(1.0385) & 1.9889(1.0064) & 1.9399(0.9618) \\
Local-PCA& 1.8296(0.7531) & 1.8091(0.7211) & 1.7364(0.7662) & 1.7423(0.7476) & 1.7822(0.7958) & 1.8130(0.7553) \\
KDE& 4.5179(1.5050) & 4.4408(1.4786) & 4.3820(1.4143) & 4.3258(1.4975) & 4.3213(1.4728) & 4.2659(1.4698) \\
LOG-KDE& 1.9323(0.9683) & 1.8667(0.9296) & 1.8461(0.9188) & 1.8018(0.8514) & 1.7974(0.8288) & 1.7955(0.8242) \\
Mfit& 2.4032(1.1298) & 2.3861(1.1216) & 2.3477(1.0988) & 2.2902(1.1229) & 2.2852(1.0957) & 2.2225(1.0391) \\
Moving-LS& {\bf 1.5518}(1.0340) & \underline{1.4685}(0.7865) &\underline{1.5185}(0.9541) &\underline{1.5515}(1.2535) & \underline{1.5083}(1.1097) & \underline {1.4567}(0.9147) \\ \hline\hline
\multicolumn{7}{r}{\tiny Numbers in bold and underlined are the best and second-best results for each column's setting, respectively.}
\end{tabular}
\end{table*}

It is computationally expensive to directly fitting the manifold in $\mathbb{R}^{4096}$ 
using any manifold learning algorithm.
To reduce the dimensionality, we approximate the original dataset 
by a $D$-dimensional subspace such that $\iota_i\approx Ux_i$, where $U\in\mathbb{R}^{4096\times D}$ denotes $D$ principal eigenvectors of
$S=\frac{1}{m}\sum_{i}\iota_i\iota_i^T$ and $x_i=U^\top\iota_i\in\mathbb{R}^D$. 
We could fit a $d$-dimensional manifold in $\mathbb{R}^D$ and map such a manifold to the pixel space $\mathbb{R}^{4096}$ via $U$. In what follows, we fix $D=20$. We construct the noisy dataset as $\{\iota'_i=Ux'_i\}_{i=1}^{m}
$ with $x'_i$ given by
\$
x_i'=x_i+\epsilon_i,\quad \epsilon_i\sim N(0,\sigma^2 I_D),\quad i=1,\ldots,m.
\$
Next, we recover a $5$-dimensional manifold in $\mathbb{R}^D$ by applying all six manifold learning algorithms to $\{x'_i\}_{i=1}^m$. For each method and each sample, we use the nearest $K$ samples to fit the local manifold.  In addition, we use the Gaussian kernel in \eqref{equ:62} and for each $y$, we set the bandwidth $h$ as $\|y-x_{i_K}\|_2$, where $x_{i_K}$ is the $K$-th nearest neighbour of $y$. 
Figure~\ref{fig:6} visualizes 16 denoised images using these manifold learning methods with $K=60$. 

To evaluate the performance, we use the following mean squared error and standard deviation of the error between the solution and the real image:
\begin{gather}
{\rm MSE} = \frac{1}{m}\sum_{i=1}^m \|\iota_i -\widehat{\iota}_i\|_2^2,\notag\\
 {\rm SD} = \sqrt{\frac{1}{m}\sum_{i=1}^m (\| \iota_i-\widehat{\iota}_i\|_2^2-{\rm MSE})^2},\notag
\end{gather}
where  $\hat\iota_i=U\hat x_i$ for all $i$ and $\hat x_i$ is the $i$-th fitted point in $\mathbb{R}^D$. Figure~\ref{cryo_K_delta} displays the MSE of RQMF-E with different $K$ and $\delta$. It shows that RQMF-E achieves the best performance when $K=48$ and $\delta=50$. 
To compare different methods, we collect the MSEs of all methods with varying $K$ in Table~\ref{tab:3}.
It can be seen that RQMF-E outperforms other methods in terms of MSE and SD in most settings. If we focus on the best performance of each method, RQMF-E is again favored with an error 1.3482 when $K=48$.
Therefore, by taking the curvature information into account, RQMF-E exhibits a stronger expressive ability than other methods and thus achieves better denoising performance on this dataset.

\section{Concluding Remarks}\label{sec:7}
This paper proposes a quadratic matrix factorization framework to learn the structure of the observed data. 
We develop an alternating minimization algorithm to solve the non-convex quadratic matrix factorization problem as well as a regularized version.
Theoretical convergence properties are established. We also present a novel transformation-based parameter-tuning method for regularized quadratic matrix factorization and intuitively argue its advantages over naively tuning the original regularization parameter. 
Furthermore, we apply the proposed methods to manifold learning problems.
We demonstrate the superiority of the proposed method numerically in a synthetic manifold learning dataset and two real datasets, i.e., the MNIST handwritten dataset and a cryogenic electron microscopy dataset.


There are several interesting directions for future research. First, our work and most related works assume the intrinsic dimension $d$ is known {\it a priori}, while this information is often not available in practice. Thus, it remains an important question to estimate the intrinsic dimensionality $d$ under the quadratic matrix factorization framework. It would also be interesting to characterize the impact if $d$ is misspecified. Second, the noises in the signal-plus-noise model may be heavy-tailed or even adversarial, so it is important to develop robust algorithms. Third, non-negative constraints are widely used in linear matrix factorization  \citep{lee2000algorithms,gillis2015exact}. It is interesting to study how non-negative constraints can be used in QMF to enhance performance.

\ifCLASSOPTIONcaptionsoff
  \newpage
\fi



%



\bibliographystyle{unsrt}
\bibliography{bibfile.bib}

%

\newpage

\begin{IEEEbiographynophoto}{Zheng Zhai}
received his Ph.D. degree from the School of Mathematical Sciences, Zhejiang University, and he is currently working as a postdoc researcher at the Department of Statistical Sciences, University of Toronto. His current research interests include unsupervised learning and manifold learning.
\end{IEEEbiographynophoto}

\begin{IEEEbiographynophoto}{Hengchao Chen}
is currently pursuing his Ph.D. degree in statistics at the University of Toronto. His research interests include matrix factorization and manifold learning.
\end{IEEEbiographynophoto}


\begin{IEEEbiographynophoto}{Qiang Sun}
received his Ph.D. from the University of North Carolina at Chapel Hill. He was an associate research scholar before he joined the University of Toronto as a faculty member. He is currently an associate professor of statistics at the University of Toronto. His research lies in the intersection of statistics, optimization, and learning. 
\end{IEEEbiographynophoto}




\end{document}


%
\title{Supplementary Material to `Quadratic Matrix Factorization with Applications to Manifold Learning'}

\author{Zheng~Zhai,
        Hengchao~Chen,
        and~Qiang Sun}

\maketitle
%
%
%
%


\section{Proof of Proposition \ref{prop:1}}
\begin{proof}\label{proof:1}
Suppose $\Phi'=Z\Phi+u{\bf 1}_m^T$ for some invertible matrix $Z\in\mathbb{R}^{d\times d}$ and $u\in\mathbb{R}^d$. Fix $R$, and we first show the existence of $R'$ such that the equality $\ell(R',\Phi')=\ell(R,\Phi)$ or the stronger condition $R'T(\Phi')=RT(\Phi)$ holds. If $T(\Phi')=MT(\Phi)$ for some constant matrix $M\in\mathbb{R}^{r\times r}$ with $r=\frac{2+3d+d^2}{2}$, then taking $R'=RM$ will conclude this part of the proof. 

It turns out that the relationship $T(\Phi')=MT(\Phi)$ does hold for some constant matrix $M\in\mathbb{R}^{r\times r}$. Recall that  $T(\tau_1,\ldots,\tau_m)=[\xi(\tau_1),\ldots,\xi(\tau_m)]$ by definition. To show $T(\Phi')=MT(\Phi)$ for some constant matrix $M\in\mathbb{R}^{r\times r}$, it suffices to show that $\xi(Z\tau+u)=M\xi(\tau)$ holds for all $\tau\in\mathbb{R}^d$. By definition of $\xi(\cdot)$, each element of $\xi(Z\tau+u)$ is a second-order polynomial of $\{\tau_{[i]}\}_{i=1}^d$. The coefficients only depend on $Z$ and $u$, but they are independent of $\tau$. Together with the definition of $\xi(\tau)$, it implies that $\xi(Z\tau+u)=M\xi(\tau)$ holds for some matrix $M$ independent of $\tau$.

For the remaining part of the proposition, we utilize a dual argument. Since for any $R$, there always exists $R'$ such that $\ell(R',\Phi')=\ell(R,\Phi)$, it is easy to show that $\min_{R}\ell(R,\Phi')\leq\min_{R}\ell(R,\Phi)$. On the other hand, since $Z$ is invertible, we know $\Phi=Z^{-1}\Phi'-Z^{-1}u{\bf 1}_m^T$ and thus $\min_{R}\ell(R,\Phi)\leq\min_{R}\ell(R,\Phi')$. Combining,  $\min_{R}\ell(R,\Phi')=\min_{R}\ell(R,\Phi)$.
\end{proof}

\section{Proof of Lemma \ref{lma:1}}
\begin{proof}
    Recall the definition of $T(\Phi)$. The term $M=T(\Phi)T(\Phi)^T+\lambda JJ^T$ can be rewritten as
    \begin{align}
        M&=
        \begin{bmatrix}
        {\bf 1}_m & \Phi^T & \Psi(\Phi)^T
        \end{bmatrix}^T
        \begin{bmatrix}
        {\bf 1}_m & \Phi^T & \Psi(\Phi)^T
        \end{bmatrix}+\lambda JJ^T\notag\\
        &=A_1+A_2+A_3,
    \end{align}
    where 
    \begin{align}
        A_1&=\frac{\beta}{1+\beta}\begin{bmatrix}
        {\bf 1}_m & \Phi^T & {\bf 0}
        \end{bmatrix}^T
        \begin{bmatrix}
        {\bf 1}_m & \Phi^T & {\bf 0}
        \end{bmatrix}+\frac{\lambda}{2}JJ^T,\\
        A_2&=\begin{bmatrix}
        \frac{1}{\sqrt{1+\beta}}{\bf 1}_m & \frac{1}{\sqrt{1+\beta}}\Phi^T & \sqrt{1+\beta}\Psi(\Phi)^T
        \end{bmatrix}^T\\
        &\begin{bmatrix}
        \frac{1}{\sqrt{1+\beta}}{\bf 1}_m & \frac{1}{\sqrt{1+\beta}}\Phi^T & \sqrt{1+\beta}\Psi(\Phi)^T
        \end{bmatrix},\\
        A_3&=\frac{\lambda}{2}JJ^T-\beta\begin{bmatrix}
        {\bf 0} & {\bf 0} & \Psi(\Phi)^T
        \end{bmatrix}^T
        \begin{bmatrix}
        {\bf 0} & {\bf 0} & \Psi(\Phi)^T
        \end{bmatrix}.
    \end{align}
    Here we take $\beta=\frac{\lambda}{\max\{2\sigma_1^2(\Psi(\Phi)),1\}}>0$. By definition of $J$ and the choice of $\beta$, we know $A_3$ is a positive semi-definite matrix. Observe that $A_2$ is also a positive semi-definite matrix. Finally, since $\Phi\Phi^T=I_d$ and $\Phi{\bf 1}_m=0$, $A_1$ is a positive definite matrix with $\sigma_{\min}(A_1)\geq\min\{\beta/(1+\beta),\lambda/2\}>0$. Combining, we know $M$ is a positive definite matrix with $\sigma_{\min}(M)\geq\sigma_{\min}(A_1)\geq\min\{\beta/(1+\beta),\lambda/2\}>0$. Since $J$ is a matrix of full column rank, the matrix $J^TM^{-1}J$ is also positive definite.
\end{proof}

\section{Proof of Proposition \ref{prop:7}}

\begin{proof}
Since $R'\in\Omega$, we have $\lambda\|R'J\|_{\rF}^2=0$ and thus $\ell_{\lambda}(R',\Phi')=\|X-RT(\Phi)\|_{\rF}^2$. By definition of $(R^*,\Phi^*)$, we have
\begin{equation}
\begin{aligned}
    \|X-R^*T(\Phi^*)\|_{\rF}^2\leq&\ell_\lambda(R^*,\Phi^*)\\
    \leq&\ell_\lambda(R',\Phi')=\|X-R'T(\Phi')\|_{\rF}^2,
\end{aligned}
\end{equation}
which concludes the proof.
\end{proof}

\section{Proof of Proposition \ref{prop:2}}
\begin{proof}
    Taking derivative of $s(\lambda)$ with respect to $\lambda$, we have
    \begin{equation}
        s'(\lambda)=2\langle \widetilde R(\lambda)JJ^T,\widetilde R'(\lambda)\rangle.\label{equ:42}
    \end{equation}
    Recall that $\widetilde R(\lambda)=XT(\Phi)^T(T(\Phi)T(\Phi)^T+\lambda
    JJ^T)^{-1}$. It follows by direct calculation that
    \begin{equation}\label{equ:43}
    \begin{aligned}
        \widetilde R'(\lambda)
        =&-XT(\Phi)^T(T(\Phi)T(\Phi)^T+\lambda JJ^T)^{-1}JJ^T\\
        &\hspace{1.5cm} (T(\Phi)T(\Phi)^T+\lambda JJ^T)^{-1}\\
        =&-\widetilde R(\lambda)JJ^T(T(\Phi)T(\Phi)^T+\lambda JJ^T)^{-1}.
    \end{aligned}
    \end{equation}
    Substituting \eqref{equ:43} into \eqref{equ:42}, we get
    \begin{equation}
    \begin{aligned}
        s'(\lambda)&=-\textnormal{tr}\left(\widetilde R(\lambda)JJ^T(T(\Phi)T(\Phi)^T+\lambda JJ^T)^{-1}JJ^T\widetilde R(\lambda)^T\right)\\
        &\leq0.
        \end{aligned}
    \end{equation}
    If $\lambda>0$, then by Lemma 4,  $J^T(T(\Phi)T(\Phi)^T+\lambda JJ^T)^{-1}J$ is positive definite. This, together with $s(\lambda)>0$, would imply that $s'(\lambda)<0$.  
    Further, taking derivative of $s'(\lambda)$ with respect to $\lambda$, we obtain
    \begin{equation}
    \begin{gathered}
        s''(\lambda)=3\textnormal{tr}(\widetilde R (\lambda)J\left(J^T(T(\Phi)T(\Phi)^T+\lambda JJ^T)^{-1}J\right)^2\\
        J^T\widetilde R(\lambda)^T)\geq 0,
      \end{gathered}
    \end{equation} 
    Similar to $s'(\lambda)$, the strict inequality $s''(\lambda)>0$ holds if $\lambda>0$ and $s(\lambda)>0$.
\end{proof}